\documentclass[11pt]{article}

\usepackage[margin=1.in]{geometry}

\usepackage{caption}
\usepackage{wrapfig}
\usepackage[T1]{fontenc}
\usepackage[utf8]{inputenc}
\usepackage{authblk}
\usepackage{color}
\usepackage{amsmath, amssymb, mathtools, latexsym}
\usepackage{psfrag,amsfonts,latexsym,amsthm,,amscd,url}
\usepackage{graphicx}
\usepackage{bm}
\usepackage{appendix}
\usepackage{subcaption}
\usepackage{booktabs}
\usepackage[colorlinks=true,linkcolor=blue,citecolor=blue, linktocpage=true]{hyperref}
\allowdisplaybreaks[4]

\newcommand{\eps}{\varepsilon}

\renewcommand{\phi}{\varphi}

\theoremstyle{plain}
\newtheorem{theorem}{Theorem}[section]

\theoremstyle{definition}

\theoremstyle{remark}
\newtheorem{remark}[theorem]{Remark}

\newcommand{\BE}{\begin{equation}}
\newcommand{\EE}{\end{equation}}
\newcommand{\BEN}{\begin{equation*}}
\newcommand{\EEN}{\end{equation*}}
\newcommand{\BAL}{\begin{align}}
\newcommand{\EAL}{\end{align}}
\newcommand{\BAN}{\begin{align*}}

\begin{document}

\title{Mathematical analysis of singularities in the diffusion model under the submanifold assumption}

\author[1]{Yubin Lu\thanks{These authors contributed equally to this work.}}
\author[2]{Zhongjian Wang\thanks{These authors contributed equally to this work.}\thanks{Corresponding author: zhongjian.wang@ntu.edu.sg}}
\author[3]{Guillaume Bal}

\affil[1]{Department of Applied Mathematics, Illinois Institute of Technology}
\affil[2]{Division of Mathematical Sciences, School of Physical and Mathematical Sciences, Nanyang Technological University}
\affil[3]{Department of Mathematics, Statistics and CCAM, University of Chicago}

\renewcommand*{\Affilfont}{\small\it}
\renewcommand\Authands{ and }
\date{\today}
\maketitle

\begin{abstract}
This paper concerns the mathematical analyses of the diffusion model in machine learning. The drift term of the backward sampling process is represented as a conditional expectation involving the data distribution and the forward diffusion. The training process aims to find such a drift function by minimizing the mean-squared residue related to the conditional expectation.  Using small-time approximations of the Green's function of the forward diffusion, we show that the analytical mean drift function in DDPM and the score function in SGM asymptotically blow up in the final stages of the sampling process for singular data distributions such as those concentrated on lower-dimensional manifolds, and are therefore difficult to approximate by a network. To overcome this difficulty, we derive a new target function and associated loss, which remains bounded even for singular data distributions. We validate the theoretical findings with several numerical examples.
\end{abstract}




\section{Introduction}

The field of generative models has emerged as a powerful tool for building continuous probabilistic models that generate new samples from given discrete datasets. Furthermore, by accounting for the joint distribution of observable (condition, query) and target variables, these models offer a flexible and efficient way to generate samples based on queries.
Such generative models have been applied across a wide range of disciplines, including computer vision \cite{image}, speech signal processing \cite{speech}, natural language processing \cite{NLP}, and natural sciences \cite{natrualSci}. Recent advances in generative models, including the popular variational autoencoder (VAE) \cite{VAE}, generative adversarial network (GAN) \cite{GAN}, flow-based model \cite{KRmap}, and DeepParticle model \cite{DeepParticle}, have demonstrated their ability to solve diverse problems across different domains. These models share a common feature: { they directly model a push-forward map from easy-to-sample distribution to an unknown distribution driven by data through neural networks (including the composition of invertible functions).}

In contrast to these direct constructions, another type of generative model links distributions through one-parameter continuous deformations ({parameter $t\in[0,T]$ or $[0,\infty)$  in the following up various models)}. This approach has a long history in the mathematical literature. {A straightforward example is the Langevin Monte Carlo (LMC), where we start from arbitrary distribution solving the following SDE,
\begin{align}\label{FK-SDE}
dX_t=v(t,X_t)dt+\sqrt{2D(t,X_t)}dW_t,
\end{align}
where $(v,D)=(-\nabla \log p_{data}, I)$. From arbitrary initial distribution $X_0\sim p_0$, LMC \eqref{FK-SDE} solves until $t$ is sufficiently large. Then the distribution of $X_t$ continuously deforms from arbitrary distribution $p_0$ to $p_{data}$ as $t\to\infty$.

Without analytic expression to $p_{data}$, a number of constructive data-driven approaches to $(v,D)$ have been considered in the literature, e.g. \cite{song2019generative,block2020generative,wang2021deep,NeuralODE,NFs}. Once a model of $(v,D)$ is derived either analytically or from data, the stochastic differential equation (SDE) integrator can be used to numerically solve the SDE from initial to terminal time (some $T>0$ or $\infty$), thereby interpreting it as a generative model.

Among these constructions, the diffusion models have attracted huge attention due to their well-known performance in practical applications.} Inspired by non-equilibrium thermodynamics \cite{Sohl-Dickstein}, \cite{DDPM} proposed denoising diffusion probabilistic models (DDPMs), a class of latent variable models, as an early diffusion model. Later, \cite{SGM} unified several earlier models through the lens of stochastic differential equations and proposed score-based generative models (SGMs). The backward process (generation of new samples) can be interpreted as solving Eq.~\eqref{FK-SDE} with a \textit{tweak} that reverts the notation of time and the initial distribution after the tweak is assumed to be a standard normal distribution \cite{anderson1982reverse}. \cite{survey1} and \cite{survey2} provided literature reviews from different perspectives.

Despite its success, the sampling process for diffusion models is extremely slow and the computational cost is high. In DDPMs \cite{DDPM}, for instance, 1000 steps are typically needed to generate samples. Several works have attempted to accelerate the sampling process \cite{DPMsolver1, DPMsolver2, zhang2022fast,DistillDiffModel,jolicoeur-martineau2021gotta,zhang2023gddim,Distillation1,Distillation2}. In addition, \cite{zhang2022fast} pointed out that there were dramatically different performances in terms of discretization error and training error when they trained the score function of SGM on different datasets.
{On theoretical side, several recent works have shown the convergence bound for diffusion models, for instance \cite{bortoli2021diffusion,lee2022convergence,schen2023sampling,hchen2023improved}. It is then natural to seek justifications for the gap between practical costly performance and theoretical convergence guarantee.
}

{In fact, the aforementioned theoretical results require the score function to be approximated well by a neural network, in the $L^2$ or $L^\infty$ sense. While in practice, when dealing with singular data distributions (for example, distributions supported on a lower-dimensional manifold), the assumptions on well-approximated score may not be tractable. To this end, our current work aims to provide a mathematical characterization of the score and adaptions of the diffusion models based on the mathematical analysis.}

\textbf{Main Contributions of this work:} (1) We rigorously characterize the singularity of the score function when the target distribution is supported on an embedded submanifold.
This justifies the singularities observed in numerical experiments in \cite{zhang2022fast}, and mentioned in \cite{schen2023sampling,hchen2023improved,lee2022convergence,bortoli2022convergence,chen2023score,nachmani2021non}.
(2) We propose a parameterization of score function {(conditional expectation model, CEM)} bypasses the problem of modeling a singular function and significantly improves the efficiency of the training process. (3) {Based on the asymptotic analysis of the singularities, we configure the proposed parameterization (CEM), over  the loss weighting during training and the sampling schedule.}

\noindent\textbf{Discussion of Related Work.}
{
It is worth mentioning \cite{pidstrigach2022scorebased} which established a lower bound that satisfies $E\|\nabla\log(p_t)\|\gtrsim\frac{1}{\sqrt{t}}$ under similar manifold assumption. The expectation is taken with the measure defined by the forward process.  Our theory provides a larger asymptotic which scales as $\frac{1}{t}$. The difference is due to our analysis is in the strong sense, which is more practical when applying deterministic neural networks to approximate scores and is aligned with the universal approximation theorems. Hence our theories has the potential to be extended to further analysis.}
{We are also aware of a parallel work \cite{chen2023score} which provided a similar optimal lower bound $\frac{1}{t}$ while assuming the support of data is on a linear subspace. In comparison, our work is more general and practical as we only require a locally defined low-dimensional representation of distribution.}

{In the context of bypassing singularities, \cite{bortoli2022convergence} and \cite{zhang2022fast} explored the effect of singular data distributions, and showed the convergence in Wasserstein distance given a locally (for $t>0$) well-approximated score. The estimate is based on the early stopping technique, namely, sampling at $t=\delta$ with $\delta>0$ small enough so that $p_{\delta}$ approximates $ p_0$ in Wasserstein distance. Our bounds validate the potential asymptotic scale in the assumption of score approximation in \cite{bortoli2022convergence}. Furthermore our characterization may potentially be extended to an existence result for the correct approximation of score. Then equipped with the early stopping technique in \cite{bortoli2022convergence}, one can show the convergence of diffusion models with more general and easier to validate assumptions, which we will leave as a future direction.}

{From the experimental perspective of diffusion models, DDPM and SGM conducted different weighting functions for training the neural network to estimate the score function. \cite{Karras,DistillDiffModel,heitz2023iterative} and \cite{campbell2022a} proposed similar loss functions or training objectives for improving the performance of diffusion models, especially at the stage $t\to0$.  Despite similar configurations of models, our characterization of score function as small $t$, provides a solid theoretical justifications for these configurations.}

\section{Background}\label{sec:background}
As a class of probabilistic generative models, diffusion models are used to sample from a target probability distribution.  In order to build a foundational framework for discussion, we briefly review two existing diffusion models and their connections in this section.

Given a $d$-dimensional target distribution $p_{data}$ and a random variable $X_0\sim p_{data}$, the general idea of diffusion models is to add noise to $X_0$ step by step such that $X_T$ is an easy-to-sample standard normal distribution. Subsequently, a reverse diffusion process is used to sample from $p_{data}$.
The generative task of diffusion models is done by solving an SDE for the form \eqref{FK-SDE} defined backward in time from $T$ to $0$. In general, there is no closed-form expression for the reversion and we usually learn it from available data.

\subsection{DDPM}\label{DDPM}
As a main class of diffusion models, DDPMs \cite{DDPM} learn a distribution $p_{\theta}$ that approximates $p_{data}$ as follows.  We start with the forward process, denoted by $X_k$ and $X_0\sim p_{data}$.
We gradually add Gaussian noise to the data with a schedule of $K$ steps at $\beta_1,\ldots,\beta_K$:
\begin{align}\label{DDPM_forward}
p(X_{1:K}|X_0):=\prod\limits_{k=1}^{K}p(X_k|X_{k-1})
\end{align}
and
\begin{align}
    p(X_k|X_{k-1}):=\mathcal{N}(\sqrt{1-\beta_k}X_{k-1},\beta_k I_d).
\end{align}
Denoting $\alpha_k:=1-\beta_k$ and $\bar{\alpha}_k:=\prod\limits_{s=1}^{k}\alpha_s$, we may recast it as:
\begin{align}
    X_{k+1} &= \sqrt{1-\beta_k}X_k+\sqrt{\beta_k}\epsilon,\nonumber\\
    &= \sqrt{\alpha}_kX_k+\sqrt{1-\alpha}_k\epsilon, \quad\epsilon\sim\mathcal{N}(0,I_d).
\end{align}
A notable property of the forward diffusion process is that
\begin{equation}
    p(X_k|X_0)=\mathcal{N}(\sqrt{\bar{\alpha}_k}X_0,(1-\bar{\alpha}_k)I_d),
\end{equation}
which implies that the data is converted to a standard Gaussian distribution as $\bar{\alpha}_k$ converges to $0$.

For the generative sampling task, we construct a backward process, denoted by $\widetilde{X_{0:K}}$, such that $\forall k$, $\widetilde{X_k}$ shares the same marginal distribution as $X_k$. As a starting point, $\widetilde{X_K}$ follows $\mathcal{N}(0,I_d)$, which is easy to sample. Then we iteratively find the conditional distribution
\begin{align}
p_{\theta}(\widetilde{X_{k-1}}|\widetilde{X_k}):=\mathcal{N}(\widetilde{X_{k-1}};\mu_{\theta}(\widetilde{X_k},k),\Sigma_{\theta}(\widetilde{X_k},k)),
\end{align}
 where $(\mu_{\theta},\Sigma_{\theta})$ can be learned from data evolving according to the forward diffusion process \eqref{DDPM_forward}.

\cite{DDPM} proposed $\Sigma_{\theta}(\widetilde{X_k},k)=\beta_k I_d$ and
\begin{equation}
    \mu_{\theta}(\widetilde{X_k},k)=\frac{1}{\sqrt{\alpha_k}}\big( \widetilde{X_k}-\frac{\beta_k}{\sqrt{1-\bar{\alpha}_k}}\epsilon_{\theta}(\widetilde{X_k},k)\big),
\end{equation}
where $\epsilon_{\theta}$ is modeled by a neural network. We then obtain samples from the distribution $p_\theta(\widetilde{X_{k-1}}|\widetilde{X_k})$ by computing
\begin{equation}
    \widetilde{X_{k-1}}=\frac{1}{\sqrt{\alpha_k}}\big( \widetilde{X_k}-\frac{\beta_k}{\sqrt{1-\bar{\alpha}_k}}\epsilon_{\theta}(\widetilde{X_k},k)\big)+\sqrt{\beta_k}\widetilde{N_k},
\end{equation}
where $\widetilde{N_k}\sim\mathcal{N}(0,I_d)$.

For the reverse diffusion process, \cite{DDPM} proposed finding the best trainable parameters $\theta$ by optimizing the variational lower bound:
\begin{equation}
    L:=E_p[-\log \frac{p_{\theta}(X_{0:K})}{p(X_{1:K}|X_0)}].
\end{equation}
They also found a simplified loss function, which improves sample quality:
\begin{equation}\label{Loss_simple}
    L_{simple}(\theta):=E_{k,X_0,\epsilon} \|\epsilon-\epsilon_{\theta}(\sqrt{\bar{\alpha}_k}X_0+\sqrt{1-\bar{\alpha}_k}\epsilon,k)\|^2.
\end{equation}
\subsection{Score-based generative model (SGM)}\label{SGM}
DDPM may be viewed as an SGM inferred from  discretizations of stochastic differential equations (SDEs) \cite{SGM}. The general idea of an SGM is to transform a data distribution into a known base distribution by means of an SDE, while the reverse-time SDE is used to transform the base distribution back to the data distribution. The forward process can be written as follows:
\begin{equation}\label{score-based_forward}
    dX_t=h(X_t,t)dt+g(t)dW_t,\qquad X_0\sim p_{data},
\end{equation}
where $W_t$ is a Brownian motion and $X_T\sim p_T$ approximates the standard normal distribution for large value of time $T$. The corresponding reverse-time SDE $\widetilde{X}_t$ shares the same marginal distribution as the forward process $X_t$ and hence gives a pair of $(v,D)$ in  Eq.~\eqref{FK-SDE} modulo the reversing of the direction of time. It can be written as
\begin{equation}\label{score-based_backward}
    d\widetilde{X_t}=\big[h(\widetilde{X_t},t)-g(t)^2\nabla_{\widetilde{X_t}}\log p_t(\widetilde{X_t})\big]dt+g(t)d\widetilde{W_t},
\end{equation}
where $p_t$ is the solution of the Fokker-Planck equation for the forward SDE \eqref{score-based_forward}, see \cite{anderson1982reverse}. \cite{SGM} proposed learning the {\em score} function $\nabla_{X}\log p_t(X)$ by minimizing the score-matching loss function:
\begin{equation}\label{score_loss}
    E_{t,X_0,X_t}\Big[ \lambda(t)\| S_{\theta}(X_t,t)-\nabla_{X_t}\log p_t(X_t|X_0)\|^2\Big],
\end{equation}
where $S_{\theta}(X_t,t)$ is a time-dependent score-based model, and $\lambda(t)$ is a positive weighting function.

\subsection{Training and sampling issues for diffusion models}
Diffusion models consist of two processes: a forward process and a reverse-time process. The forward process is used to transform the target distribution into a normal distribution. The forward process is explicitly given and does not require training. In contrast, the reverse-time process is used to restore the target distribution from the normal distribution, with coefficients that are not explicitly known and may be approximated by training.

{Training of DDPM seeks $\theta$ by minimizing Eq.~\eqref{Loss_simple}. For SGM, the evaluation of the score function $\nabla_{X}\log p_t(X)$ is not explicit. \cite{SGM} discuss several approximations including Gaussian transition kernel $p_t(X_t|X_0)$ or using sliced score matching. In the next section, we provide an expression \eqref{global_training} as the loss function arising from the continuous solution of the Fokker-Planck equation.}

In addition to the issue of accessibility of $\nabla_{X}\log p_t(X)$, the score function $\nabla_{X}\log p_t(X)$ in Eq.~\eqref{score_loss} may be complex and exhibit local structures and singularities, in particular near $t=0$.
This was pointed out by \cite{SGM2,SGM,zhang2022fast}, where they proposed experimental ways to deal with the singularity of the score function $\nabla_{X}\log p_t(X)$.

In the next section, we will elaborate on accessibility and regularity issues from a more mathematical point of view.

\section{ A mathematical perspective on generative diffusion models}
In this section, we first present the unified framework for DDPM and SGM based on considering the backward-generating process as an SDE solver. %
Then by exploring the analytical solution of the Fokker-Planck equation, we represent the score function as a natural conditional expectation. {
We are aware existed literature on the framework linking diffusion models and SDEs  \cite{DMandNF,SGM,SGM2} and the conditional expectation representation \cite{TweedieFormula,pidstrigach2022scorebased} which is known as Tweedie's formula. While these derivations, especially \eqref{global_S}, provide a solid background and clear intuition for the following up Section \ref{sec:singularities} which rigorously shows the existence of singularities and Section \ref{sec:newScore} which presents a parametrization, different from DDPM and SGM. }

\subsection{Unifying the time framework in DDPM and SGM}\label{sec:framework}
Inspired by SGM, the forward process of diffusion models may be viewed as a discretization of the following $d-$dimensional OU process:
\begin{align}\label{forward_process}
    dX_t=-\frac{1}{2}X_tdt+dW_t.
\end{align}
Consider a $K$ partition of the time interval $[0,T]$, $t_0=0<t_1<t_2< \cdots<t_K=T$, where $t_{k+1}-t_{k}=\Delta t_k=-\log(1-\beta_k)$ (equivalently $\beta_k=1-e^{-\Delta t_k}$). Then at time step $t_k$,
\begin{align}
    X_{t_{k+1}}&=X_{t_{k}} e^{-\frac{\Delta t_k}{2}}+\int_{t_k}^{t_{k+1}}e^{-\frac{t_{k+1}-s}{2}}dW_s\nonumber\\
        &\sim X_{t_{k}} \sqrt{1-\beta_k} +\sqrt{\beta_k} N_k, \quad N_k\sim\mathcal{N}(0,I_d),
\end{align}
which coincides with the forward diffusion process in the diffusion models literature, e.g., \cite{DDPM, SGM}. Compared to the preceding section, note that $\alpha_k=e^{-\Delta t_k}$, $\bar{\alpha}_k=e^{-t_k}$.

The advantage of the continuous model is twofold. First, we can represent the distribution of Eq.~\eqref{forward_process} at time $t$ by Eq.~\eqref{dist_forward_process}
\begin{align}\label{dist_forward_process}
X_t\sim X_0e^{-\frac{t}{2}}+\sqrt{1-e^{-t}}N,\quad \text{where}\quad N\sim\mathcal{N}(0,I_d).
\end{align}
Second, we can easily estimate the time necessary to convert the real data distribution to normal distribution. Empirically, we take the final time $T>10$ which leads to the fraction of the initial data in $X_T$ to be less than $\exp(-5)=0.0067$.

The backward (sampling) process follows the reverse-time SDE \cite{anderson1982reverse},
\begin{align}\label{backward_process}
    d\widetilde{X_t}=-(\frac{1}{2}\widetilde{X_t}+\nabla_X\log p(\widetilde{X_t},t))dt+d\widetilde{W}_t,
\end{align}
where $p_t$ is a forward Kolmogorov equation of Eq.~\eqref{forward_process} with initial data distribution $p_{data}$ and $\widetilde{W}_t$ is a backward Brownian motion independent of $W_t$. Then $X_t$ and $\widetilde{X_t}$ share the same marginal distribution.

A key observation is that $p$, the law of the OU process, has an analytical solution \cite{evans2010partial},
\begin{align}\label{marginal_density}
    p(X,t)=\frac{1}{Z} \int\exp\Big(-\frac{\|X-X_0 e^{-\frac{t}{2}}\|^2}{2(1-e^{-t})}\Big)p_{data}(X_0)dX_0,
\end{align}
where $Z$ is a normalizing factor that depends on $t$ and $d$. The global score function $S$ may then be interpreted as a conditional expectation, namely,
\begin{align}
    &S(X,t)=-\nabla_X\log p(X,t)=-\frac{\nabla_X p(X,t)}{p(X,t)}\nonumber\\
    =&\frac{\frac{1}{Z} \int\frac{X-X_0e^{-t/2}}{1-e^{-t}}\exp\Big(-\frac{\|X-X_0 e^{-t/2}\|^2}{2(1-e^{-t})}\Big)p_{data}(X_0)dX_0,}{\frac{1}{Z} \int\exp\Big(-\frac{\|X-X_0 e^{-t/2}\|^2}{2(1-e^{-t})}\Big)p_{data}(X_0)dX_0}\nonumber\\
    =&E_{{X_0|X_t}}[\frac{X_t-X_0 e^{-t/2}}{1-e^{-t}}|X_t=X],\label{global_S}
\end{align}
where $X_t$ follows the forward process \eqref{forward_process}. The conditional expectation in \eqref{global_S} can be interpreted as follows. Starting from $X_0$ following $p_{data}$, the forward process solving \eqref{forward_process} gives the base distribution $X_t$. Given the observation of $X_t$ as $X$, $\frac{X_t-X_0 e^{-t/2}}{1-e^{-t}}$ follows a posterior distribution. Taking the expectation of the posterior gives the analytical expression of $S$.

\begin{remark}
    The score function defined in our paper, $S(X,t)=-\nabla_X\log p(X,t)$, differs slightly from the convectional form, $\nabla_X\log p(X,t)$. In other words, the sign of the score function depends on the direction of the schedule.
\end{remark}
\begin{remark}
By standard Markov property, for $t'<t$,
\begin{align}\label{global_S1}
    S(X,t)=E_{{X_{t'}|X_t}}\Big[\frac{X_t-X_{t'} e^{-(t-{t'})/2}}{1-e^{-(t-{t'})}}|X_t=X\Big].
\end{align}
\end{remark}

\noindent\textbf{Training.}
In general, $S(x,t):\mathbb{R}^d\times [0,T]\to \mathbb{R}^d$ is a very high dimensional function and hence lacks global approximation. Leveraging the properties of the conditional expectation for fixed $t$, $S(X,t)$ is the optimizer of the following mean-squared prediction error functional (justification see Section \ref{app:training-goal}),
\begin{align}
J(S)=E_{X_0,X_t}\Big\|\frac{X_t-X_0 e^{-t/2}}{1-e^{-t}}-{S}(X_t,t)\Big\|^2. \label{training-goal}
\end{align}
By assigning a weight for the $t$-variable, the training process of \cite{SGM} is generalized as seeking a network-represented function $S_\theta$ that minimizes,
\begin{align}
   &E_{X_0,N\sim \mathcal{N}(0,I_d),t}\nonumber\\
   &\Big[\lambda(t) \big\|\frac{N}{\sqrt{1-e^{-t}}}-S_\theta(X_0e^{-t/2}+\sqrt{1-e^{-t}}N,t)\big\|^2\Big]. \label{global_training}
\end{align}
We remark that we use samples of $(X_0,t,N)$ for the evaluation of the integral in Eq.~\eqref{global_training} and samples of $t$ do not necessarily follow the same schedule as those of the backward process.

\noindent\textbf{Sampling.} There is no general closed-form  solution for the backward process \eqref{backward_process} and so we employ splitting schemes,
\begin{equation}\label{global_sampling}
\begin{cases}
\overline{X_{t_{k+1}}}=\widetilde{X_{t_{k+1}}}-\Delta t_{k}S_\theta(\widetilde{X_{t_{k+1}}},t_{k+1})\\
    \widetilde{X_{t_k}}=e^{\Delta t_k/2}\overline{X_{t_{k+1}}}+\sqrt{1-e^{-\Delta t_k}}\widetilde{N}_k
    \end{cases}
\end{equation}
where $\widetilde{N}_k\sim\mathcal{N}(0,I_d)$.
\begin{remark}
    The training and sampling process exactly coincide with the aforementioned SGM, i.e., learning the score function $S(x,t)$ with $L^2$-norm. It is also related to DDPM as follows. Taking $\epsilon_\theta(x,t)=\sqrt{1-e^{-t}}S_\theta(x,t)$, the loss function \eqref{global_training} becomes,
\begin{align}
  &\int_t \frac{\lambda(t)}{1-e^{-t}}E_{X_0,N\sim \mathcal{N}(0,I_d)} \Big\|N-\epsilon_\theta(X_0e^{-t/2}+\sqrt{1-e^{-t}}N,t)\Big\|^2dt\nonumber\\
  \approx &\sum_k \frac{\lambda(t_k)}{1-\bar{\alpha}_k} E_{X_0,N\sim \mathcal{N}(0,I_d)} \Big\|N-\epsilon_\theta(X_0\sqrt{\bar{\alpha}_k}+\sqrt{1-\bar{\alpha}_k}N,t_k)\Big\|^2.
  \end{align}
And the sampling process \eqref{global_sampling},
\begin{align}
    \widetilde{X_{t_k}}=&e^{\Delta t_k/2}(\widetilde{X_{t_{k+1}}}-\Delta t_{k}S_\theta(\widetilde{X_{t_{k+1}}},t_{k+1}))+\sqrt{1-e^{-\Delta t_k}}\widetilde{N}_k\\
     =&\frac{1}{\sqrt{\alpha_k}}(\widetilde{X_{t_{k+1}}}-\frac{\Delta t_{k}}{\sqrt{1-\bar{\alpha}_k}}\epsilon_\theta(\widetilde{X_{t_{k+1}}},t_{k+1}))+\sqrt{\beta_k}\widetilde{N}_k\nonumber\\
     \approx&\frac{1}{\sqrt{\alpha_k}}(\widetilde{X_{t_{k+1}}}-\frac{1-\alpha_k}{\sqrt{1-\bar{\alpha}_k}}\epsilon_\theta(\widetilde{X_{t_{k+1}}},t_{k+1}))+\sqrt{\beta_k}\widetilde{N}_k.
\end{align}
For the approximation in the last line, we use $1-\alpha_k=\beta_k\approx \Delta{t_k}$. This scheme coincides with DDPM using $L_{simple}$ loss function \eqref{Loss_simple}.
\end{remark}

\subsection{Singularity of the score function}\label{sec:singularities}
In the previous subsection, we showed the conventional training process aimed to approximate the conditional expectation function $S(X,t)=E_{{X_0|X_t}}[\frac{X_t-X_0e^{-t/2}}{1-e^{-t}}|X_t=X]$ or $\epsilon(X,t)=\sqrt{1-e^{-t}}S(X,t)$. However, such functions potentially exhibit singularities near $t=0$, which corresponds to the last few steps of the sampling process. For example, if $X_0$ follows a single point distribution, then $S(X,t)=\frac{X-X_0e^{-t/2}}{1-e^{-t}}$ while $\epsilon(X,t)=\frac{X-X_0e^{-t/2}}{\sqrt{1-e^{-t}}}$. {It requires a large number of latent variables in a common network to model such a blow-up as $t\to0$.}

An $n$-dimensional sub-manifold is denoted by $\Omega$, where $\Omega\subset\mathbb{R}^d$ and $n<d$. To characterize such asymptotics for most general datasets, we made the following assumptions over point $X$ in the backward (sampling) process and data distribution $p_{data}$.

\noindent\textbf{(H1) Uniqueness Assumption} Fixing point $X$, we denote the $y_X$ on $\Omega$ as the unique point that minimize the distance between $X$ and $\Omega$, i.e. $y_X=\mathrm{argmin}_{y\in\Omega}\|y-X\|$ is uniquely defined.
\begin{remark}
    We note that the uniqueness assumption (H1) could potentially be relaxed. However, we use this assumption here to provide a clearer proof of Theorem \ref{Thm1} in order to avoid a lengthy discussion about the classification of boundaries of data manifold.
\end{remark}

\noindent\textbf{(H2) Subspace Assumption}
Let $B_{\varepsilon}=\{y\in \Omega: \|y-X\|<\|y_X-X\| + \varepsilon\}$, which is decreasing set series as $\varepsilon\to 0$. We assume there exists $0<\varepsilon_0\ll1$, such that for $y\in B_{\varepsilon_0}$, there exists a local coordinate chart, $z\to y(z)\in B_{\varepsilon_0}\subset\Omega$, under which $p_{data}$ is assumed to have a locally defined smooth density function in form of,
\begin{align}\label{localpdata}
    p_{data}(y)=\hat{\rho}(z)|J(z)|\delta_{y(z)\in\Omega},
\end{align}
where $J$ is the Jacobian of local coordinate transformation, and the size of $J$ is corresponding to the dimension of low-dimensional variable $z$, denoted as $n$. {$\delta_{{y(z)\in\Omega}}$ denotes a $\delta$-type function supported on $n$ dimensional subspace and vanishes when $y$ is away from $\Omega$.}\footnote{{Another explanation of Eq.~\eqref{localpdata} is $\forall y\in \Omega$, within neighbourhood of $y$, $ p_{data}$ is absolutely continuous with respect to ($n$-dimensional) Lebesgue measure on $\Omega$.}}
In addition, we assume within $y(z)\in B_{\varepsilon_0}$, $\hat{\rho}(z)$ is continuous and bounded,
\begin{align}\label{localboundc}
    0<\rho_0\leq\hat{\rho}(z)|J(z)|\leq \rho_1 <\infty.
\end{align}
Under these assumptions, we state the first key theorem of this work as follows.
\begin{theorem}\label{Thm1}(Singularity of the score functions)
Let $X\in\mathbb{R}^d\backslash\Omega$ and data distribution $p_{data}$ satisfy (H1) and (H2). Then, the score function $S(X,t)$ blows up as $t\to0$, and more precisely, satisfies
\begin{align}\label{Slimit}
S(X,t) = \frac{X-y_X}{t} \big(1+o(1)\big).
\end{align}
\end{theorem}
We will leave the proof of this theorem and the following theorem in Section \ref{sec:proof}.
\begin{remark}
     The singularity of the score function demonstrated in Theorem \ref{Thm1} is in a strong, point-wise sense, rather than in a weak sense, for instance $L_2$ norms under distribution of forward process as discussed in \cite{pidstrigach2022scorebased}. This point-wise singularity may offer valuable insights for score modeling and approximations.
\end{remark}

In contrast to results in Theorem~\ref{Thm1}, there are situations where the target functions ($\epsilon$ and $S$) remain bounded as $t$ approaches zero for specific distributions $p_{data}$.
\begin{theorem}\label{thm2}(Regularity of the score function) Assuming the data distribution $X_0$ has the following form of probability density function,
\begin{align}
    \rho=\rho_0*\mu_1,
\end{align}
where $\rho_0$ can be {the PDF of any random variable\footnote{Technically, we assume the random variable ensures all related expectations are well-defined.}, denoted as $\hat{X}_0$,} and $\mu_1$ is the PDF of a normal distribution with variance $\sigma^2>0$. Then fixing $X$,
\begin{align}
    -\nabla_X\log p(X,t)=E_{{\hat{X}_0|X_t}}\Big[\frac{X_t-e^{-t/2}\hat{X}_0}{\sigma^2e^{-t}+1-e^{-t}}|X_t=X\Big].
\end{align}
The score function $-\nabla_X\log p(X,t)$ remains bounded when the support of $\rho_0$ is compact.
\end{theorem}
Theorem \ref{thm2} provides a possible explanation for why samples of the DDPM and SGM seem to be randomly perturbed away from the possible local support of the data distribution manifold. As discussed in Theorem \ref{Thm1}, the theoretical value of the target function becomes unbounded as $t$ approaches $0$, which is not expressible by most network configurations. The loss function that relies on the target function becomes unbounded too. The model turns out to learn a bounded function instead of a singular function, which corresponds to learning a polluted data distribution $\rho$ instead of $\rho_0$ ($p_{data}$). A sample from $\rho=\rho_0*\mu_1$ can be viewed as adding independent Gaussian noise to a sample from the original distribution $p_{data}$.

Summarizing the above, forcing the network, upper bounded by $\frac{1}{\sigma^2}$, to learn the model $S$ or $\epsilon$ from data supported on a low-dimensional geometry, turns out to add i.i.d. Gaussian noise in each dimension with variance $\sigma^2$ to the original data.
\begin{remark}
    On the other side Theorem \ref{thm2} provides a sufficient condition for uniform bounded score function. This may facilitate a well-approximated score function that ensures the convergences of the diffusion model.
\end{remark}
\subsection{A new model based on conditional expectation}\label{sec:newScore}
 To avoid such pollution,  we propose the conditional expectation model (CEM) to respect the singularities. Note that,
\begin{align}\label{S_decompose}
    S(X,t)=\frac{X}{1-e^{-t}}-\frac{e^{-t/2}}{1-e^{-t}}E_{{X_0|X_t}}[X_0|X_t=X].
\end{align}
Denoting $E_{{X_0|X_t}}[X_0|X_t=X]$ as $f(X,t)$, we know for fixed $t$ that $f(\cdot,t)$ minimizes the following functional,
\begin{align}
    J(f)= E_{X_0,X_t}(\|X_0-f(X_t,t)\|)^2.
\end{align}
This justifies defining a new loss function for training $f_\theta$ as
\begin{align}\label{newscore_loss}
    E_{X_0,X_t,t}\Big[\lambda(t) \big\|X_0-f_\theta(X_t,t)\big\|^2\Big],
\end{align}
where $\lambda(t)>0$ is a time-dependent weighing function that remains free for the user to choose.
\begin{remark} Similar objective functions as \eqref{newscore_loss} were mentioned in \cite{DDPM} very briefly, while we provide theoretical analysis (Section \ref{sec:singularities}) and numerical results (Section \ref{sec:comparison} for justification.
\end{remark}
    A good choice of $\lambda$ is to align the training process for each $t$. While the analytical value is inaccessible without knowledge of data distribution, in practice, we employ $\lambda(t)=(e^t-1)^{-1}$. This is inspired by an analysis of $L_{simple}$ in DDPM discussed as follows. {In DDPM, the target of the model reads,}
\begin{align}\label{eq:optimal_DDPM}
    \epsilon_\theta(X,t)=E_{{X_0|X_t}}[\frac{X_t-X_0e^{-t/2}}{\sqrt{1-e^{-t}}}|X_t=X].
\end{align} Fixing $t$, the loss function is lower bounded, i.e.,
\begin{align}\label{eq:lower_lambda_DDPM}
    E_{X_t}(\epsilon_\theta(X_t,t)-\frac{X_t-X_0e^{-t/2}}{\sqrt{1-e^{-t}}})^2
    \geq E_{X_t}Var[\frac{X_t-X_0e^{-t/2}}{\sqrt{1-e^{-t}}}|X_t],
\end{align}
where the equality holds when \eqref{eq:optimal_DDPM} holds. Note that
\begin{align}
    Var[\frac{X_t-X_0e^{-t/2}}{\sqrt{1-e^{-t}}}|X_t]=\frac{e^{-t}}{1-e^{-t}}Var(X_0|X_t).
\end{align}
The uniform weight in $L_{simple}$ implies a lower bound in the right-hand side of Eq.~\eqref{eq:lower_lambda_DDPM} is \textbf{assumed} to be independent of $t$. With the same assumption, we have
\begin{equation}
   E_{X_t}Var(X_0|X_t)\propto (e^{t}-1),
\end{equation}
which in turn gives $\lambda(t)=(e^t-1)^{-1}$.

\noindent\textbf{Sampling process.}
After training for $f_\theta$, the closed form solution of the backward process \eqref{backward_process} remains unknown. Thus, we still need to use numerical SDE solvers to construct a generative model of $p_{data}$. Using \eqref{S_decompose}, we consider the following replacement in the sampling scheme \eqref{global_sampling},
\begin{align}
    S_\theta(X,t) =  \frac{X}{1-e^{-t}}-\frac{e^{-t/2}}{1-e^{-t}}f_\theta(X,t).
\end{align}
Even after re-directing the network to model a bounded function $f$, the drift term in the backward process, $-\frac{1}{2}X-\nabla_X \log p(X,t)$, may still be of order $\mathcal{O}(\frac{1}{t})$ near $t=0$; see \eqref{Slimit}. The training schedule should be adapted accordingly.  A natural choice is to match the drift scale with a single time step. At the time $t_k$ for $k>1$, we consider the scale of changes due to the drift,
\begin{align}
    (t_{k}-t_{k-1})\frac{1}{t_{k}}:= \gamma_k.
\end{align}
Minimizing $\gamma_k$ for all $k>1$, we arrive at the following exponential schedule,
\begin{align}\label{exp_law}
t_k=t_1(1-\gamma)^{1-k}
\end{align}
where $\gamma=1-(\frac{T}{t_1})^{\frac{1}{K-1}}$.
\begin{remark}
    Though the scale of drift indicated in Eq.~\eqref{Slimit}, i.e., $\mathcal{O}(\frac{1}{1-e^{-t}})$, only applies when $t$ is near $0$. For $t\gg0$, we still use exponential schedule \eqref{exp_law} to reduce the time of network evaluation.
\end{remark}

{
\noindent\textbf{Case of general forward process} The unified framework in Section \ref{sec:framework} may be extended to the score-based generative models with general forward processes \eqref{score-based_forward}. More precisely, the marginal density follows,
\begin{align*}
    p(X,t)=\frac{1}{Z} \int K(X,X_0)p_{data}(X_0)dX_0,
\end{align*}
where $K(X,X_0)$ is the transition probability (Green's kernel) of the forward processes \eqref{score-based_forward}. The representation \eqref{global_S} turns to
\begin{align*}
    -\nabla_X\log p(X,t) = E_{{X_0|X_t}}[-\frac{\nabla_X K(X_t,X_0)}{K(X_t,X_0)}|X_t=X],
\end{align*}
and the loss function follows the misfit functional in the similar form as \eqref{global_training},
\begin{align*}
    J(S) = E_{X_0,X_t}\Big\|-\frac{\nabla_X K(X_t,X_0)}{K(X_t,X_0)}-S(X_t,t)\Big\|^2.
\end{align*}
In this regard, with tractable (e.g. with analytic expression) kernel $K$ and forward process $X_t|X_0$, one can formulate the training loss function as conditional expectation and backward process as \eqref{score-based_forward} with approximated $-\nabla_X \log p_t$. While the singularity theories, e.g. Theorem \ref{Thm1}, relies on small time asymptotic of heat/OU kernel \cite{evans2010partial} which is directly induced by the explicit formula, for instance derivation for \eqref{global_S}. Estimate of $\frac{\nabla_X K(X_t,X_0)}{K(X_t,X_0)}$ for general $K$ is non-trivial, we will leave it as a future direction.
}

\subsection{Proofs  }\label{sec:proof}\quad

\noindent\textbf{Proof of Theorem \ref{Thm1}}

 The score function has the following representation
\begin{align}\label{S_decompse}
    S(X,t)&=E_{{X_0|X_t}}[\frac{X_t-X_0e^{-\frac{t}{2}}}{1-e^{-t}}|X_t=X]=\frac{g(X,t)}{1-e^{-t}},
\end{align}
\begin{align}\label{G_decompse}
    g(X,t) &= E_{{X_0|X_t}}[X-X_0e^{-\frac{t}{2}}|X_t=X] = \frac{\int_{\Omega}(X-ye^{-\frac{t}{2}})e^{-\frac{\|X-ye^{-\frac{t}{2}}\|^2}{2(1-e^{-t})}}p_{data}(y)dy}{\int_{\Omega}e^{-\frac{\|X-ye^{-\frac{t}{2}}\|^2}{2(1-e^{-t})}}p_{data}(y)dy}
\end{align}
  With a fixed $\varepsilon>0$,  we decompose $g$ into two parts,
\begin{align}\label{G_decompse2}
    g(X,t)= \frac{\int_{B_{\varepsilon}}(X-ye^{-\frac{t}{2}})e^{-\frac{\|X-ye^{-\frac{t}{2}}\|^2}{2(1-e^{-t})}}p_{data}(y)dy}{\int_{\Omega}e^{-\frac{\|X-ye^{-\frac{t}{2}}\|^2}{2(1-e^{-t})}}p_{data}(y)dy} + \frac{\int_{\Omega\backslash B_{\varepsilon}}(X-ye^{-\frac{t}{2}})e^{-\frac{\|X-ye^{-\frac{t}{2}}\|^2}{2(1-e^{-t})}}p_{data}(y)dy}{\int_{\Omega}e^{-\frac{\|X-ye^{-\frac{t}{2}}\|^2}{2(1-e^{-t})}}p_{data}(y)dy}.
\end{align}
 By definition of $B_{\varepsilon}$, for $y\in\Omega\backslash B_{\varepsilon}$,
 \begin{align}
     \|X-ye^{-\frac{t}{2}}\|\geq e^{-\frac{t}{2}}\|X-y\|-(1-e^{-\frac{t}{2}})\|X\|\geq e^{-\frac{t}{2}}(\|X-y_X\|+\epsilon)-(1-e^{-\frac{t}{2}})\|X\|=:C_{t,\varepsilon}.
 \end{align}
 For $y\in B_{\varepsilon}$,
  \begin{align}
     \|X-ye^{-\frac{t}{2}}\|\leq e^{-\frac{t}{2}}\|X-y\|+(1-e^{-\frac{t}{2}})\|X\|\leq e^{-\frac{t}{2}}(\|X-y_X\|+\epsilon)+(1-e^{-\frac{t}{2}})\|X\|=:D_{t,\varepsilon}.\label{dis_est2}
 \end{align}

We claim that the second term of \eqref{G_decompse2} converges to zero as $t\to0$ (with fixed $\varepsilon$) since
\begin{align}
\Bigg\| \frac{\int_{\Omega\backslash B_{\varepsilon}}(X-ye^{-\frac{t}{2}})e^{-\frac{\|X-ye^{-\frac{t}{2}}\|^2}{2(1-e^{-t})}}p_{data}(y)dy}{\int_{\Omega}e^{-\frac{\|X-ye^{-\frac{t}{2}}\|^2}{2(1-e^{-t})}}p_{data}(y)dy}\Bigg \| \nonumber
\leq &  \frac{\int_{\Omega\backslash B_{\varepsilon}}(\|X\|+\|y\|)e^{-\frac{C_{t,\epsilon}^2}{2(1-e^{-t})}}p_{data}(y)dy}{\int_{\Omega}e^{-\frac{\|X-ye^{-\frac{t}{2}}\|^2}{2(1-e^{-t})}}p_{data}(y)dy} \\
\leq & \frac{\int_{\Omega}(\|X\|+\|y\|)p_{data}(y)dy}{\int_{\Omega}e^{-\frac{\|X-ye^{-\frac{t}{2}}\|^2-C^2_{t,\epsilon}}{2(1-e^{-t})}}p_{data}(y)dy}.
\end{align}
Given the boundedness of the expectation of the data distribution $p_{data}$, it remains to show the denominator converges to infinity as $t\to 0$. In fact, with \eqref{dis_est2} {and Subspace Assumption (H2)} in mind,
\begin{align}
    \int_{\Omega}e^{-\frac{\|X-ye^{-\frac{t}{2}}\|^2-C^2_{t,\epsilon}}{2(1-e^{-t})}}p_{data}(y)dy&\geq\int_{B_{\varepsilon'}}e^{-\frac{\|X-ye^{-\frac{t}{2}}\|^2-C^2_{t,\epsilon}}{2(1-e^{-t})}}p_{data}(y)dy\nonumber \\&\geq  \int_{y(z)\in B_{\varepsilon'}}e^{-\frac{D^2_{t,\varepsilon'}-C^2_{t,\varepsilon}}{2(1-e^{-t})}}\hat{\rho}(z)|J(z)|dz.\label{small_ball_bound}
\end{align}
With $t$ sufficient small, say $t<t_0$ such that $\frac{\varepsilon}{2}>2(e^{\frac{t_0}{2}}-1)\|X\|$, we set $\varepsilon'=\frac{\varepsilon}{2}-2(e^{\frac{t_0}{2}}-1)\|X\|>0$ so that
$\forall\ 0<t<t_0$,
\begin{align}
C_{t,\varepsilon}^2-D_{t,\varepsilon'}^2=&\Big(e^{-\frac{t}{2}}(\varepsilon-\varepsilon')-2(1-e^{-\frac{t}{2}})\|X\|\Big)e^{-\frac{t}{2}}(2\|X-y_X\|+\varepsilon+\varepsilon')\nonumber\\
=& \Bigg(e^{-\frac{t}{2}}\big(\frac{\varepsilon}{2}+2(e^{\frac{t_0}{2}}-1)\|X\|\big)-2(1-e^{-\frac{t}{2}})\|X\|\Bigg)e^{-\frac{t}{2}}(2\|X-y_X\|+\varepsilon+\varepsilon')\nonumber\\
=& \Big(e^{-\frac{t}{2}}\frac{\varepsilon}{2}+2(e^{\frac{t_0-t}{2}}-1)\|X\|\Big)e^{-\frac{t}{2}}(2\|X-y_X\|+\varepsilon+\varepsilon')\nonumber\\
\geq & e^{-t_0}\varepsilon \|X-y_X\|>0.
\end{align}
The right-hand side of \eqref{small_ball_bound} converges to infinity as $t\to 0$.

Similarly,
\begin{align}
&\frac{\int_{\Omega\backslash B_{\varepsilon}}e^{-\frac{\|X-ye^{-\frac{t}{2}}\|^2}{2(1-e^{-t})}}p_{data}(y)dy}{\int_{B_{\varepsilon}}e^{-\frac{\|X-ye^{-\frac{t}{2}}\|^2}{2(1-e^{-t})}}p_{data}(y)dy}
\leq  \frac{\int_{\Omega\backslash B_{\varepsilon}}e^{-\frac{\|X-ye^{-\frac{t}{2}}\|^2}{2(1-e^{-t})}}p_{data}(y)dy}{\int_{B_{\varepsilon'}}e^{-\frac{\|X-ye^{-\frac{t}{2}}\|^2}{2(1-e^{-t})}}p_{data}(y)dy}\nonumber\\
\leq & \frac{\int_{\Omega\backslash B_{\varepsilon}}e^{-\frac{C_{t,\varepsilon}^2}{2(1-e^{-t})}}p_{data}(y)dy}{\int_{B_{\varepsilon'}}e^{-\frac{D_{t,\varepsilon'}^2}{2(1-e^{-t})}}p_{data}(y)dy}
\leq   \frac{1}{\int_{y(z)\in B_{\varepsilon'}}e^{-\frac{D_{t,\varepsilon'}^2-C_{t,\varepsilon}^2}{2(1-e^{-t})}}\hat{\rho}(z)|J(z)|dz}= o(t).
\end{align}
So the denominator in the first term of \eqref{G_decompse2} can also be decomposed and approximated by the contribution in $B_{\varepsilon}$
\begin{align}
    \int_{\Omega}e^{-\frac{\|X-ye^{-\frac{t}{2}}\|^2}{2(1-e^{-t})}}p_{data}(y)dy= (1+o(t))\int_{B_{\varepsilon}}e^{-\frac{\|X-ye^{-\frac{t}{2}}\|^2}{2(1-e^{-t})}}p_{data}(y)dy.
\end{align}
Then when $t\to0$, we have in local coordinates \eqref{localpdata},
\begin{align}
    &\frac{\int_{B_{\varepsilon}}(X-ye^{-\frac{t}{2}})e^{-\frac{\|X-ye^{-\frac{t}{2}}\|^2}{2(1-e^{-t})}}p_{data}(y)dy}{\int_{B_{\varepsilon}}e^{-\frac{\|X-ye^{-\frac{t}{2}}\|^2}{2(1-e^{-t})}}p_{data}(y)dy} \nonumber\\
    =&\frac{\int_{y(z)\in B_{\varepsilon}}(X-y(z)e^{-\frac{t}{2}})e^{-\frac{\|X-y(z)e^{-\frac{t}{2}}\|^2}{2(1-e^{-t})}}\hat{\rho}(z)|J(z)|dz}{\int_{y(z)\in B_{\varepsilon}}e^{-\frac{\|X-y(z)e^{-\frac{t}{2}}\|^2}{2(1-e^{-t})}}\hat{\rho}(z)|J(z)|dz}\nonumber\\
    =&X-e^{-\frac{t}{2}}\frac{\int_{y(z)\in B_{\varepsilon}}y(z)e^{-\frac{\|X-y(z)e^{-\frac{t}{2}}\|^2}{2(1-e^{-t})}}\hat{\rho}(z)|J(z)|dz}{\int_{y(z)\in B_{\varepsilon}}e^{-\frac{\|X-y(z)e^{-\frac{t}{2}}\|^2}{2(1-e^{-t})}}\hat{\rho}(z)|J(z)|dz}.\label{localchartexp}
\end{align}
Taking \eqref{localboundc} into account, and realizing that $y(z)e^{-\frac t2}$ is well approximated by $y_X$ on $B_\eps$ for $t$ small,
\begin{align}
   &\Bigg\|\frac{\int_{y(z)\in B_{\varepsilon}}y(z)e^{-\frac{\|X-y(z)e^{-\frac{t}{2}}\|^2}{2(1-e^{-t})}}\hat{\rho}_t(z)|J(z)|dz}{\int_{y(z)\in B_{\varepsilon}}e^{-\frac{\|X-y(z)e^{-\frac{t}{2}}\|^2}{2(1-e^{-t})}}\hat{\rho}_t(z)|J(z)|dz}-y_X\Bigg\|\nonumber\\
   \leq&  \frac{\rho_1\int_{y(z)\in B_{\varepsilon}}\|y(z)-y_X\|e^{-\frac{\|X-y(z)e^{-\frac{t}{2}}\|^2}{2(1-e^{-t})}}dz}{\rho_0\int_{y(z)\in B_{\varepsilon}}e^{-\frac{\|X-y(z)e^{-\frac{t}{2}}\|^2}{2(1-e^{-t})}}dz}\nonumber\\
   \leq& \frac{\rho_1\int_{y(z)\in B_{\varepsilon}}\varepsilon e^{-\frac{\|X-y(z)e^{-\frac{t}{2}}\|^2}{2(1-e^{-t})}}dz}{\rho_0 \int_{y(z)\in B_{\varepsilon}}e^{-\frac{\|X-y(z)e^{-\frac{t}{2}}\|^2}{2(1-e^{-t})}}dz}\nonumber\\
   \leq & \frac{\varepsilon \rho_1}{ \rho_0}.\label{diff_to_yt}
\end{align}

Substituting back to \eqref{localchartexp} then \eqref{G_decompse} we have
\begin{align}\label{limg}
    \lim_{t\to 0}g(X,t)=X-y_X+\mathcal{O(\varepsilon)}.
\end{align}
Since choice of $\varepsilon>0$ is arbitrary, from \eqref{S_decompse} we have,
\begin{align}\label{S_limit2}
    \lim_{t\to 0} tS(X,t)=\lim_{t\to 0}\frac{ t(X-y_X)}{1-e^{-t}} =X-y_X.
\end{align}
\begin{remark}
    The above derivation may be generalized as an application of the Laplace method, which we now briefly present.  The manifold $\Omega$ is covered by charts mapping subsets to domains of Euclidean space. Consider one such chart parameterized by variables $y=y(z)$ with $z\in U\subset {\mathbb R}^n$ and Jacobian of the transformation equals $1$ to simplify. We assume that the closest point $y_X=y_X(z_X)$ for $z_X\in U$.

The Laplace method (see, e.g., \cite{erdelyi1956asymptotic})  states that an integral of the form
   \begin{align}\label{laplace_assumption}
       \int_{U} e^{-\frac1{t} \theta(z)} h(z)dz
   \end{align}is approximated by,
\begin{align}
(2\pi t)^\frac n2 \frac{h(z_0)}{|H\theta(z_0)|^{\frac12}}e^{-\frac1{t} \theta(z_0)} (1+o(1))
\end{align}
where $z_0$ is the unique point minimizing $\theta(z)$ and $|H\theta(z_0)|$ is the determinant of the positive definite Hessian of $\theta$ at $z_0$.

We start with,
    \begin{align}
   S(X,t)= \frac{\int_{U}(X-y(z)e^{-\frac{t}{2}})e^{-\frac{\|X-y(z)e^{-\frac{t}{2}}\|^2}{2(1-e^{-t})}}\hat{\rho}(z)dz}{(1-e^{-t})\int_{U}e^{-\frac{\|X-y(z)e^{-\frac{t}{2}}\|^2}{2(1-e^{-t})}}\hat{\rho}(z)dz}.
    \end{align}
    By noticing $1-e^{-t}\approx t$ when $t\to 0$ and applying the Laplace method \eqref{laplace_assumption} to the nominator and denominator with $\theta(z)=\|X-y(z)e^{-\frac{t}{2}}\|^2/2$ and $h(z)=(X-y(z)e^{-\frac{t}{2}})\hat{\rho}(z)$ and $h(z)=\hat{\rho}(z)$ correspondingly, we immediately arrive at,
    \begin{align}
    S(X,t) = \frac{X-y_X}t (1+o(1)).
    \end{align}
\end{remark}

\begin{remark}
    The uniqueness assumptions of Theorem \ref{Thm1} hold almost surely during the backward process. This is due to during the backward process, the target function (e.g. $S$ and $\epsilon$) is evaluated on an approximated distribution of the forward process, which is globally supported over $\mathbb{R}^d$ and hence almost surely outside of $\Omega$. As for the subspace assumption, it is a widely shared belief that data distribution in, e.g., human genes, climate patterns, and images, are supported on low dimensional structures \cite{tenenbaum2000global, roweis2000nonlinear, belkin2003laplacian}. We would like to further remark that the dimension $n$ in the subspace assumption is only locally defined and our result holds as long as $n<d$.
\end{remark}

\noindent\textbf{Proof of Theorem \ref{thm2}.}

Let $N_1$ and $N_2$ be independent standard normal distributions with variance $\sigma^2$ and $1-e^{-t}$ correspondingly. Denoting $\hat{X}_0$ as a random variable with distribution $\rho_0$ and $X_0$ is also a random variable with distribution $\rho$. Notice that $\rho=\rho_0*\mu_1$, where $\mu_1$ is PDF of $N_1$. Therefore, we know that $X_0=\hat{X}_0+N_1$. Moreover, the solution of forward process \eqref{forward_process} is $X_t=X_0e^{-\frac{t}{2}}+\sqrt{1-e^{-t}}N$, where $N$ is a standard normal random variable. Subsequently, the equation $X_t=(\hat{X}_0+N_1e^{-\frac{t}{2}})+N_2$ is hold in distribution sense. Using this relation, we can derive

\begin{align}
    &E[e^{-t/2}X_0|X_t=X]\nonumber\\
    =&E[e^{-t/2}(\hat{X}_0+N_1)|e^{-t/2}(\hat{X}_0+N_1)+N_2=X]\nonumber\\
    =&E[e^{-t/2}\hat{X}_0|e^{-t/2}(\hat{X}_0+N_1)+N_2=X]+E[E[e^{-t/2}N_1|e^{-t/2}N_1+N_2=X-e^{-t/2}\hat{X}_0,\hat{X}_0]]\nonumber\\
    =&E[e^{-t/2}\hat{X}_0|e^{-t/2}(\hat{X}_0+N_1)+N_2=X]+E[\frac{\sigma^2e^{-t}(X-e^{-t/2}\hat{X}_0)}{\sigma^2e^{-t}+1-e^{-t}}|e^{-t/2}(\hat{X}_0+N_1)+N_2=X]\nonumber\\
    =&E[\frac{\sigma^2e^{-t}X+(1-e^{-t})e^{-t/2}\hat{X}_0)}{\sigma^2e^{-t}+1-e^{-t}}|e^{-t/2}(\hat{X}_0+N_1)+N_2=X].
\end{align}
So
\begin{align}
    -\nabla_X\log p(X,t)&=E_{{X_0|X_t}}[\frac{X_t-X_0 e^{-t/2}}{1-e^{-t}}|X_t=X]\nonumber\\
    &=E_{{\hat{X}_0|X_t}}[\frac{X_t-e^{-t/2}\hat{X}_0}{\sigma^2e^{-t}+1-e^{-t}}|X_t=X],
\end{align}
which remain bounded when the support of $\rho_0$ is compact.

\section{Experiments}
The experiments consist of five parts. First, we employ DDPM, SGM, and the proposed CEM to learn a one-dimensional supported distribution in $\mathbb{R}^2$.
By comparing the learned models with the corresponding analytic values, we show the new model outperforms DDPM and SGM by avoiding approximating singularities.
Second, we also verify that if we replace the network with its corresponding analytic expression, the sampling process gives the exact distribution. Third, we investigate the performance of the new model depending on some parameters that are decided empirically. Fourth, we apply CEM to a high-dimensional example, i.e., the MNIST dataset, and compare the performance with DDPM. Lastly, we conduct ablation studies  for the sampling schedule and the weighting function.

In all subsequent experiments, unless otherwise specified, we setup the diffusion model with $T = 10$ as the `final' time and $K=200$ uniform/non-uniform time grid points (exponential schedule \eqref{exp_law}) for training/sampling. For model training, we use $10^6$ samples with a batch size of $10^4$, and we choose \textbf{Adam} as the optimizer, where the learning rate is 0.001. The network configuration will be specified in each example.
\subsection{Comparison between SGM, DDPM and proposed CEM}\label{sec:comparison}
In the following, we compare the SGM, DDPM, and CEM on several two-dimensional target distributions.

\noindent\textbf{Line normal distribution in 2d space.} As the first example, we consider a distribution supported on a line in two-dimensional space. Precisely, the data distribution is generated by $X=(X_1,X_2)$, where $X_1\sim\mathcal{N}(0,1)$, $X_2=0$. In Appendix \ref{eg:lowDmanifold}, we derive the explicit formulas:
\begin{equation}\label{eqn:analytic_1dline}
    \begin{cases}
        S(X,t)=(X_1, \frac{X_2}{1-e^{-t}}),\\
        \epsilon(X,t)=(\sqrt{1-e^{-t}}X_1, \frac{X_2}{\sqrt{1-e^{-t}}}),\\
        f(X,t)=(X_1e^{-\frac{t}{2}}, 0).
    \end{cases}
\end{equation}


\begin{figure}[htbp]
\begin{center}
\includegraphics[bb=20 0 500 150, width=0.6\columnwidth]{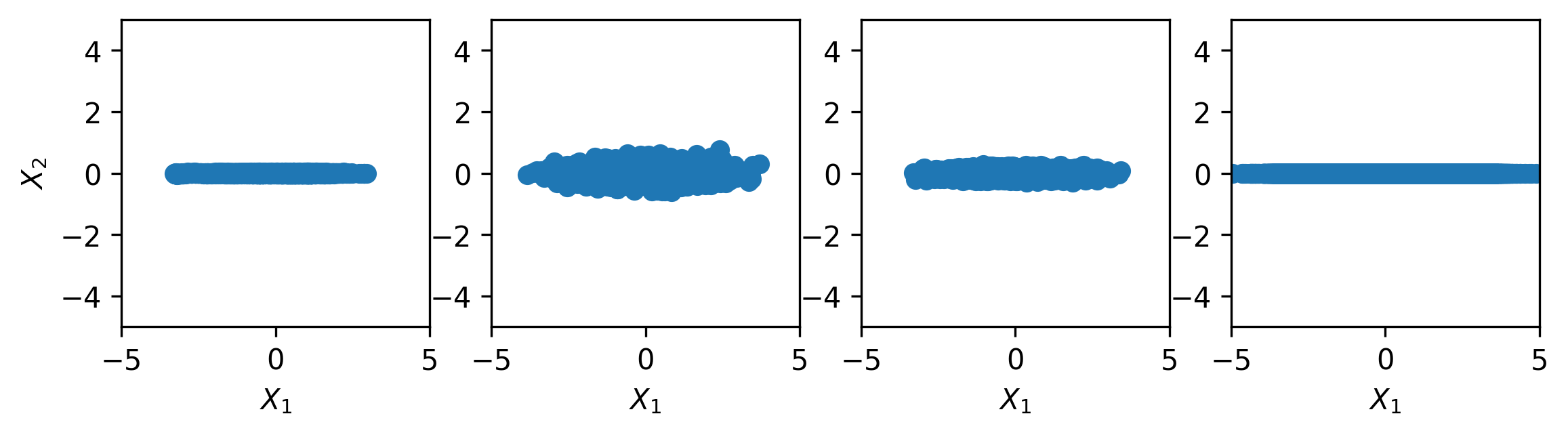}
\end{center}
\caption{1d line normal distribution in 2d space. From left to right: CEM, SGM, DDPM, and the ground truth. The network configuration is as follows: 2 hidden layers, each layer with 16 nodes, and \textbf{Tanh} as the activation function.}
\label{fig:line_normal}
\end{figure}
In Figure \ref{fig:line_normal}, we compare the distributions generated after training the CEM, SGM, and DDPM.  Our method CEM displays less pollution errors compared to DDPM and SGM. To verify that the errors originate from the poor approximation of the goal functions, we compare in Figure \ref{fig:error_line_normal_point} the estimated score function $S_{\theta}(X,t)$ with the ground truth $S(X,t)$ at a fixed point $X_{eva}=(1,-0.1)$, i.e., $e(S(X_{eva},t), S_{\theta}(X_{eva},t))=\|S(X_{eva},t)-S_{\theta}(X_{eva},t)\|$.
We similarly define $e(\epsilon(X_{eva},t),\epsilon_{\theta}(X_{eva},t))$ and $e(f(X_{eva},t),f_{\theta}(X_{eva},t))$. Notice that $X_{eva}$ is outside of the support of the distribution $\mathbb{R}\times\{0\}$ and that by \eqref{eqn:analytic_1dline}, the target score functions $S$ and $\epsilon$ exhibit singularities in the second coordinate. Correspondingly, in the left of Figure \ref{fig:error_line_normal_point}, we observe that the approximations of $f$, $S$, and $\epsilon$ are roughly correct for the first coordinate. This also verifies that the training of SGM and DDPM is indeed modeling the conditional expectation suggested in Eq.~\eqref{global_S}. On the right picture of Figure \ref{fig:error_line_normal_point}, we observe that, due to the existence of singularities, the approximations of $S$ and $\epsilon$ are incorrect and the error grows rapidly in the final steps of the sampling procedure.


\begin{figure}[htbp]
\begin{center}
\includegraphics[width=0.6\columnwidth]{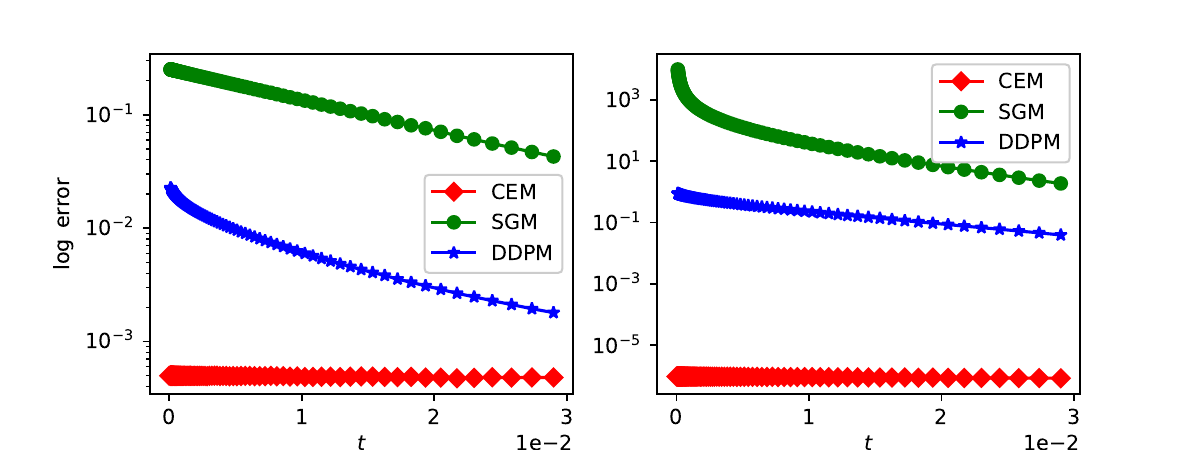}
\end{center}
\caption{Error at a fixed point $X_{eva}=(1,-0.1)$. Red, proposed CEM: $e(f,f_{\theta})$; Green, SGM: $e(S,S_{\theta})$; Blue, DDPM: $e(\epsilon,\epsilon_{\theta})$.  (Left) first component of estimated function. (Right) second component of estimated function.}
\label{fig:error_line_normal_point}
\end{figure}

Figure \ref{fig:error_line_normal} displays the $L^2$-error between the analytic formulas in \eqref{eqn:analytic_1dline} and the estimated functions $f$, $S$ and $\epsilon$ obtained during the last $100$ sampling steps in the backward process. The $L^2$ norm is defined as follows. With $S_{\theta}(X,t)$ the estimated score function, the $L^2$-error at a fixed time $t>0$ is defined by
\begin{align}
e_p(S,S_{\theta})=\int \|S(X,t)-S_{\theta}(X,t)\|^2p(X,t)dX,
\end{align} where $p(X,t)$ is distribution of the forward process \eqref{forward_process}. In practice, we solve the forward process \eqref{forward_process} to obtain the empirical distribution at time $t$ as an approximation of distribution $p$. We similarly evaluate $e_p(\epsilon,\epsilon_{\theta})$ and $e_p(f,f_{\theta})$.


\begin{figure}[htbp]
\begin{center}
\includegraphics[ width=0.6\columnwidth]{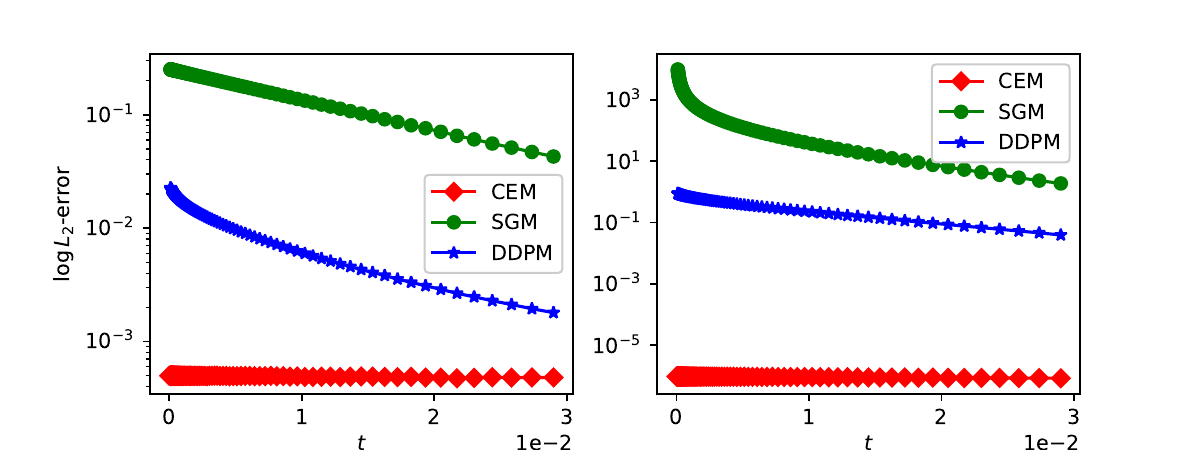}
\end{center}
\caption{$L^2$-error with distribution $p$. Red, proposed CEM: $e_p(f,f_{\theta})$; Green, SGM: $e_p(S,S_{\theta})$; Blue, DDPM: $e_p(\epsilon,\epsilon_{\theta})$. (Left) the first component of the model function. (Right) the second component of the model function.}
\label{fig:error_line_normal}
\end{figure}
Since $\epsilon=\sqrt{1-e^{-t}}S$ is of order $\mathcal{O}(\frac{1}{\sqrt{t}})$, we remark that with same configuration of network, $\epsilon$ in DDPM is better approximated than $S$ in SGM (see Figure \ref{fig:error_line_normal_point} and Figure \ref{fig:error_line_normal}). This results in less pollution in the sampling process, as shown in Figure \ref{fig:line_normal}.

\noindent\textbf{Curve distribution.} We now consider distributions with more complex geometries, and in particular a data distribution generated by  $X=(Ucos(U),Usin(U))$ where $U\sim\text{Unif}[1,13]$. In Figure \ref{fig:curve}, we compare the distributions generated by the CEM, SGM, and DDPM. The singularities near $t=0$ exhibited in Theorem \ref{Thm1} imply that errors only accumulate during the final few stages of the sampling process. The approximated stochastic dynamics primarily lead $X_t$ to a local neighborhood of the support of $p_{data}$, where most of the error is concentrated. {Furthermore, Theorem \ref{thm2} shows, that if the target data is blurred by a white noise, the score function, as well as   $\epsilon$ in DDPM, is upper bounded. So forcing network models to approximate unbounded target functions ( $\epsilon$ in DDPM and $S$ in SGM) may lead to learning a blurred version of the original distribution as shown in Figure \ref{fig:curve}.}


\begin{figure}[htbp]
\begin{center}
\includegraphics[bb=20 0 600 200, width=0.6\columnwidth]{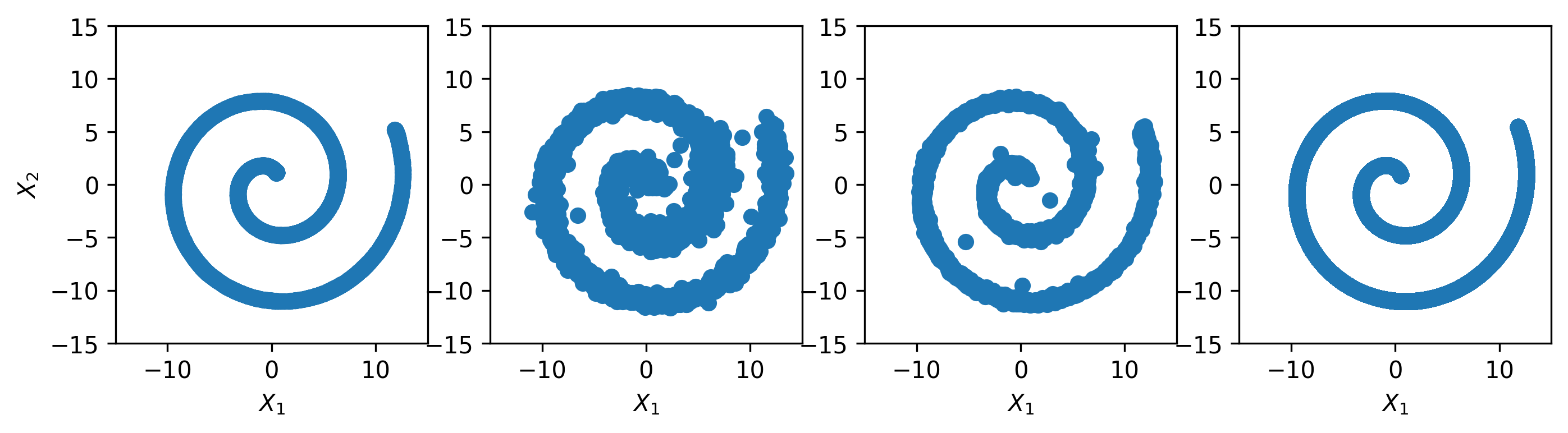}
\end{center}
\caption{Curve distribution. From left to right: the proposed CEM, SGM, DDPM, and the ground truth. The network configuration is as follows: 3 hidden layers, each layer with 64 nodes, and \textbf{Tanh} as the activation function.}
\label{fig:curve}
\end{figure}

\subsection{Replacing the network by analytical expressions} \label{replace_analytical}

In a limited number of favorable settings, the diffusion coefficients $(v,D)$ that appear in the backward sampling process may be computed explicitly leading to an equally explicit expression for the conditional expectation \eqref{global_S}. This bypasses the need to model $(v,D)$ by a neural network.

As an illustrative example, we generate $5$ points (randomly) $X_{1:5}$ and set the target data distribution $p_{data}=\sum_{i=1}^5 \delta_{X_i}$. We then obtain the following analytical expression derived in Appendix \ref{Appendix:pointcloud},
\begin{align}\label{pointcloud_expression}
E_{{X_0|X_t}}[X_0|X_t=X] =\frac{\sum_{i=1}^{5}X_{0}^{(i)}\exp{\Big(-\frac{\|X-X_{0}^{(i)}e^{-\frac{t}{2}}\|^2}{2(1-e^{-t})}\Big)}}{\sum_{i=1}^{5}\exp{\Big(-\frac{\|X-X_{0}^{(i)}e^{-\frac{t}{2}}\|^2}{2(1-e^{-t})}\Big)}}.
\end{align}
Figure \ref{fig:fivePoints_scatter} displays the backward process for 10000 samples generated by solving the backward SDE at the times $t=10,5.8718,3.2356,0.7518,0.0216,0$. Not surprisingly, the initial points sampled from a normal distribution are entirely ``absorbed" into the target five-point distribution at the final sampling step $t=0$. Table \ref{table:fivePoints_freq} counts the empirical frequencies (probability) of absorption by the five target points, which are very close to their theoretical value $0.2$.

The interpretation is then twofold. (1) With an exact model of the target function in the training process and an exact solution of the SDE \eqref{backward_process} in the sampling process, the resulting new samples accurately reproduce the original training data. This validates that the training process of the diffusion model under the framework discussed in Section \ref{sec:framework} is in fact a least square minimization that achieves optimality with conditional expectation \eqref{global_S}. (2) When explicit expressions such as \eqref{pointcloud_expression} are not available, this ideal accurate sampling of the training data can rarely be achieved in practice due to the imperfections in the neural network approximation. Only a simplified distribution is learned in practice, which enables the generalization abilities. See Figure \ref{fig:curveDeepShallow} and the next section for the reconstructions obtained in the context of a distribution for different neural nets trained to approximate $E_{{X_0|X_t}}[X_0|X_t=X]$. We point out here that there is a similar numerical experiment in \cite{peluchetti2022nondenoising}.


\begin{figure}[htbp]
\begin{center}
\includegraphics[bb=20 0 500 300, width=0.6\columnwidth]{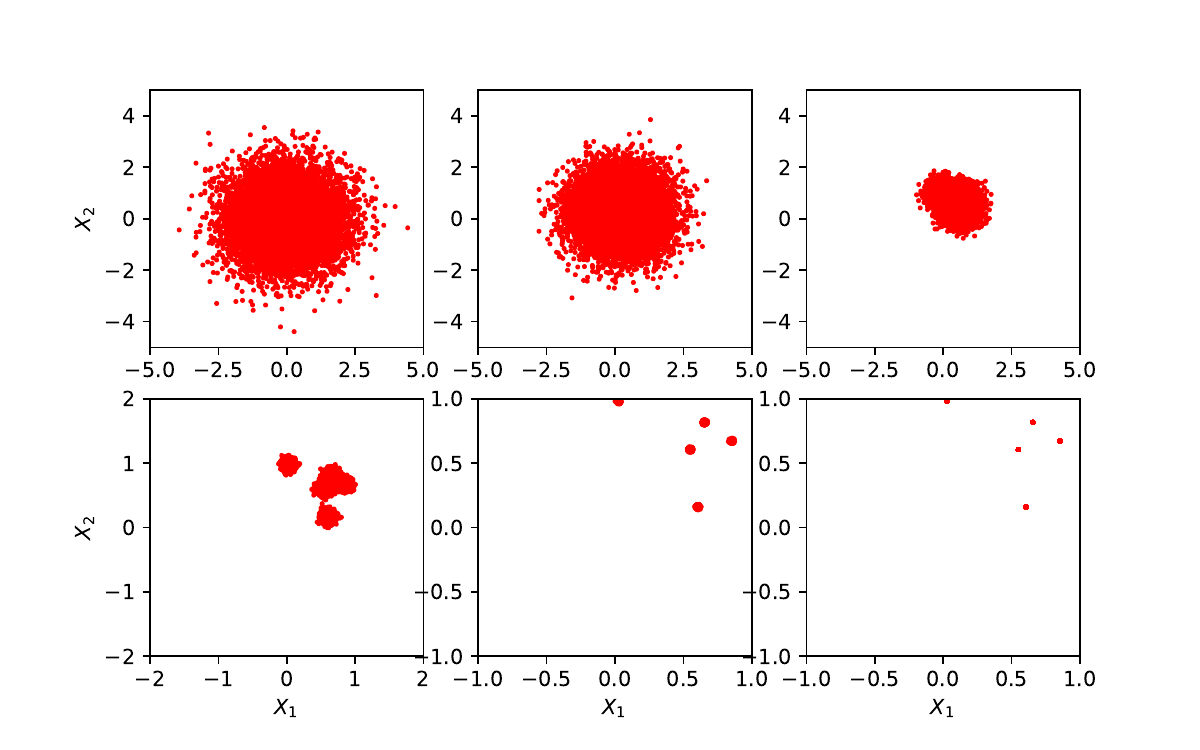}
\end{center}
\caption{Generating five-point distribution in 2d space by the analytic expression of the drift, scattering plot of sampling process for $t=10,5.8718,3.2356,0.7518,0.0216,0$.}
\label{fig:fivePoints_scatter}
\end{figure}


\begin{table}[htbp]
\caption{Sampling five-point distribution in 2d space with analytic expression. Frequency of each point.}
\label{table:fivePoints_freq}
\begin{center}
\begin{tabular}{llllll}
\multicolumn{1}{c} {POINTS}  &\multicolumn{1}{c} 1  &\multicolumn{1}{c} 2  &\multicolumn{1}{c} 3  &\multicolumn{1}{c} 4  &\multicolumn{1}{c} 5
\\ \hline \\
FREQ.      &0.2086 &0.1924 &0.2092 &0.1977 &0.1921
\end{tabular}
\end{center}
\end{table}

\subsection{Dependence on the model configuration}\label{sec:network-config}
In the following examples, we discuss the performance of the CEM under various types of configurations.

\noindent\textbf{Expressive power of network.} %
{We consider sampling an eight-point distribution in $\mathbb{R}^2$ under the new loss \eqref{newscore_loss} with  different network configurations. In Figure \ref{fig:curveDeepShallow}, we present the samples from a shallow neural network (1 hidden layer, 4 nodes each layer), a deep neural network (3 hidden layers, 64 nodes each layer), and training data (ground truth). This comparison illustrates that when the expressive power of the network is limited (through limiting number of latent variables), then CEM will generate a different distribution with lower complexity that may relate to generalization. }

{ We point out that, the above example in Figure \ref{fig:curveDeepShallow} may emulate the practical application scenario that models the score through a structured but with fewer latent variables network (e.g. the well-established U-Net in \cite{ronneberger2015u}) to generate high dimensional distributions (e.g. pictures). The new target function $f$ in CEM may have different structures compared to $\epsilon$ in DDPM and $S$ in SGM.  This may require a design of the network that differs from conventional used ones in the applications.}
Moreover, if the design of the network preserves the possible low dimensional structure of continuous data distribution, instead of the discrete samples, solving the backward process associated with the network modeled drift may \textit{generalize} the distribution of the discrete sample to the continuous one. We leave this as future work.


\begin{figure}[htbp]
\begin{center}
\includegraphics[bb=20 0 500 200, width=0.7\columnwidth]{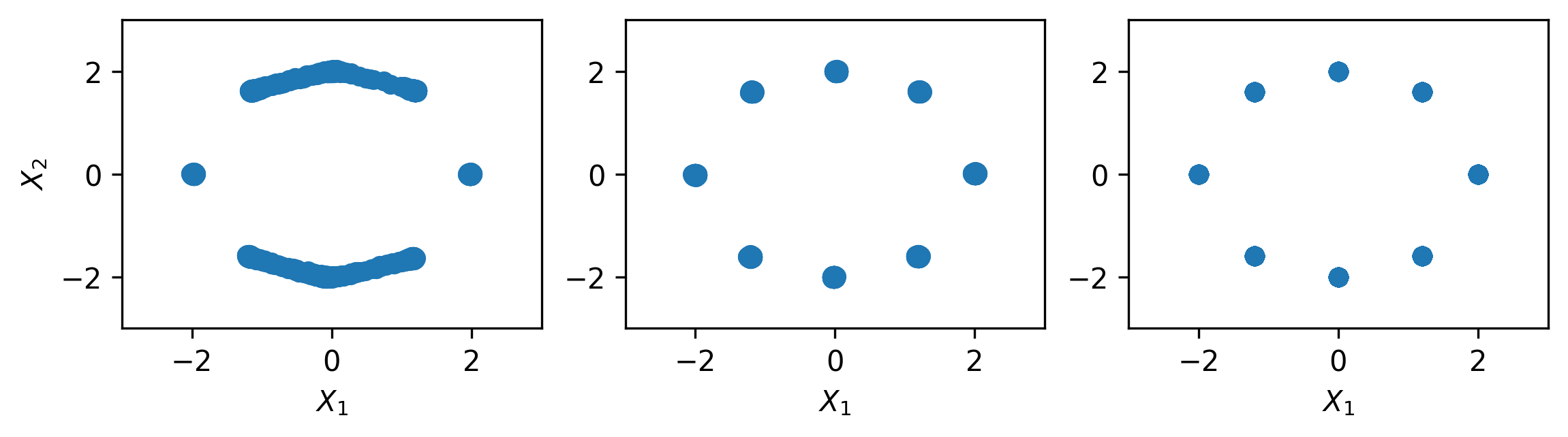}
\end{center}
\caption{{Eight-point distribution.}  From left to right: {Shallow} neural network, {deep} neural network, and the ground truth. The deep network configuration: 3 hidden layers, each layer with 64 nodes. The shallow network configuration: 1 hidden layer, each layer with {4 nodes.} \textbf{Tanh} as the activation function of both.}
\label{fig:curveDeepShallow}
\end{figure}
\noindent{\textbf{Sampling schedule $t_1$.}}
Yet another parameter to be determined in the sampling schedule proposed in Eq.~\eqref{exp_law} is $t_1$.
In Figure \ref{fig:point_cloud}, we  consider a $20$ points distribution in $\mathbb{R}^2$ and generate samples from Eq.~\eqref{global_sampling} with analytical expression for various values of $t_1$. As a splitting scheme, \eqref{global_sampling} introduces numerical errors proportional to the time step. Since the final time step is $t_1$, we can see in Figure \ref{fig:point_cloud} that smaller $t_1$ results in fewer errors in the generated distribution. In practice, we do not recommend $t_1$ to be taken as too small as this introduces numerical instabilities when computing the final drift in \eqref{global_sampling}.


\begin{figure}[htbp]
\begin{center}
\includegraphics[ width=0.8\columnwidth]{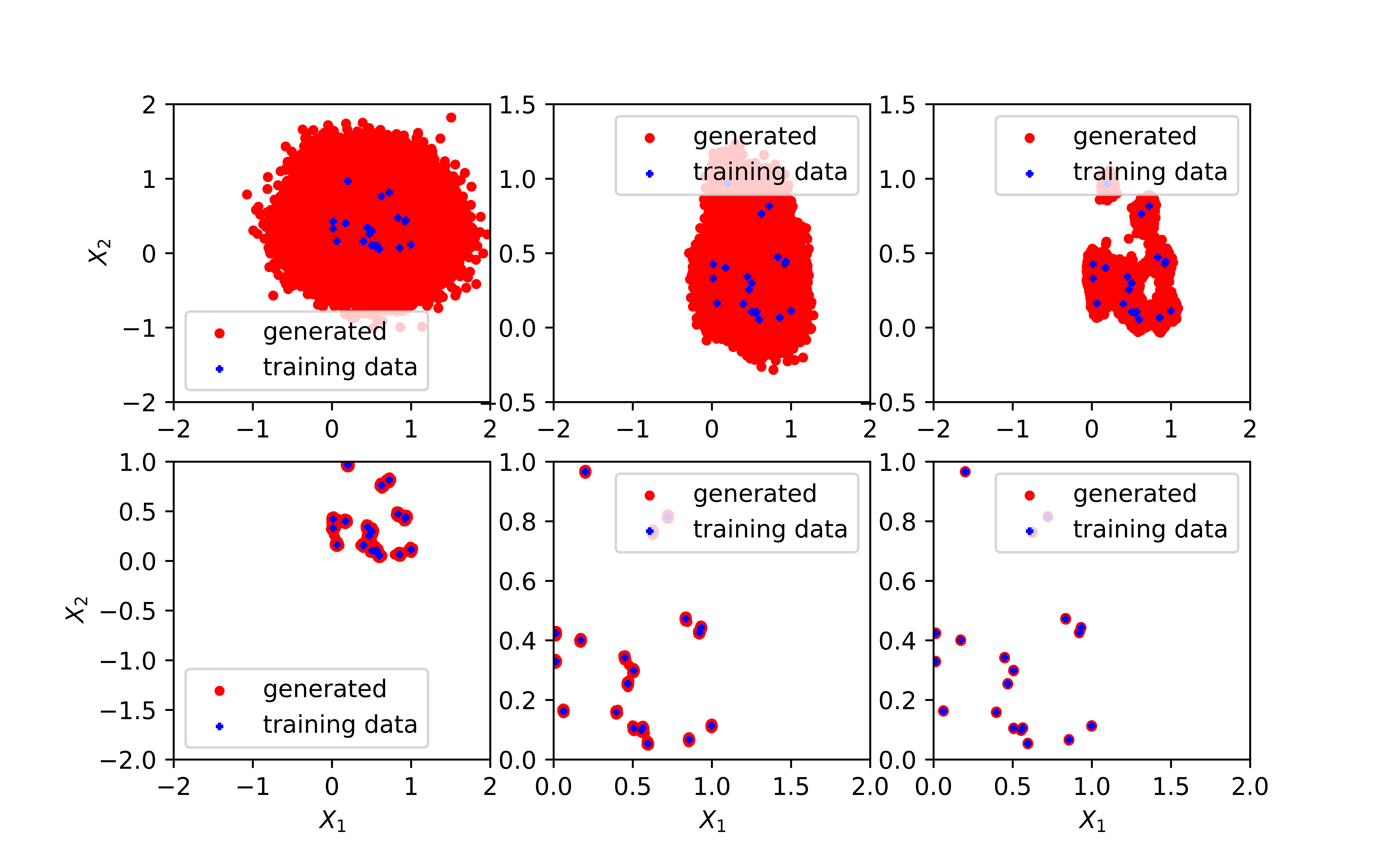}
\end{center}
\caption{Generated samples with the analytic expression of drift for various $t_1$}
    \label{fig:point_cloud}
\end{figure}
\noindent\textbf{Aligning the training process by designing $\lambda$.}
In order to further improve the  effectiveness of training, it is also important to control the variance of the loss function at different times by judicious choices of $\lambda(t)$. In section \ref{sec:newScore}, we propose to define $\lambda(t)=\frac{1}{e^t-1}$, and the previous experiments have verified its validity.
Recalling the loss function in CEM \eqref{newscore_loss}, we ensure that
\begin{align}\label{variance_fitting}
    \lambda(t)^{-1} \sim E_{X_0,X_t}\|X_0-f(X_t,t)\|^2.
\end{align}
For a given explicit expression of $f$, we can numerically estimate the right-hand side at different times by sampling the forward process $X_t$. The estimation is denoted as $\lambda_{\text{true}}(t)$, as it reflects the potential small $t$ asymptotic regime of the variance in the right-hand side of \eqref{variance_fitting}.

In Figure \ref{fig:lambda_fitted}, we revisit the case of  Section \ref{sec:comparison} with a curve target distribution $p_{data}$. The proposed $\lambda$ in Section \ref{sec:newScore} is $(e^t-1)^{-1}$. We consider $\frac{1}{\lambda_{\text{guess}}(t)}=C(e^t-1)$ with a free constant $C$ to fit the computed $\frac{1}{\lambda_{\text{true}}(t)}$. It can be seen in Figure \ref{fig:lambda_fitted} that $(e^t-1)^{-1}$ captures the correct scale despite the minor perturbations introduced by the sampling. This result is another confirmation that our proposed method CEM may greatly improve training stability in some cases.

%

%
\begin{figure}[htbp]
\begin{center}
\includegraphics[bb=0 0 200 150, width=0.3\columnwidth]{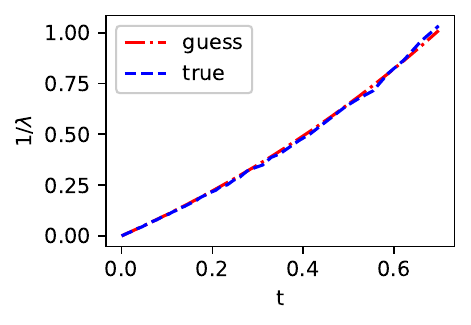}
\end{center}
\caption{Fitting $\lambda^{-1}$. True: $E_{X_0,X_t}\|X_0-f(X_t,t)\|^2$ by Monte Carlo; guess: $C(e^t-1)$ for the best constant $C$.}
\label{fig:lambda_fitted}
\end{figure}

\subsection{Application to MNIST}
In this subsection, we present the performance when applying our CEM to generate high dimensional distribution (MNIST, \ref{fig:true_MNIST}). Comparing with previous examples, we replace the densely connected net by Unet \cite{ronneberger2015u} to model  $\epsilon$ \eqref{eq:optimal_DDPM} of DDPM and $f$ of CEM separately. We apply \textbf{Adam} optimizer with a learning rate of 0.00002 and train each model with a batch size of 64 for 30 epochs. Both the forward process and the sampling process consist of 1000 steps, with a final time $T=10$.
Figure \ref{fig:CEM_MNIST_realTime} and Figure \ref{fig:DDPM_MNIST_realTime} show that the generation of CEM and DDPM correspondingly. {The FID scores of CEM and DDPM are 56.55 and 63.64, respectively.}
Limited by computing resources, this preliminary numerical result validates the potential sample generation capability of CEM for high dimensional distributions and shows the advantage of CEM over the original DDPM. {We agree that the performance of CEM in other large data sets may not be comparable with DDPM, see also discussion in Section \ref{sec:network-config}. We will leave the correct/fine-tuned network configuration of CEM as a future direction.}


\begin{figure}[htbp]
  \centering
  \begin{subfigure}[b]{0.32\textwidth}
    \includegraphics[bb=0 0 1300 1000, width=\columnwidth]{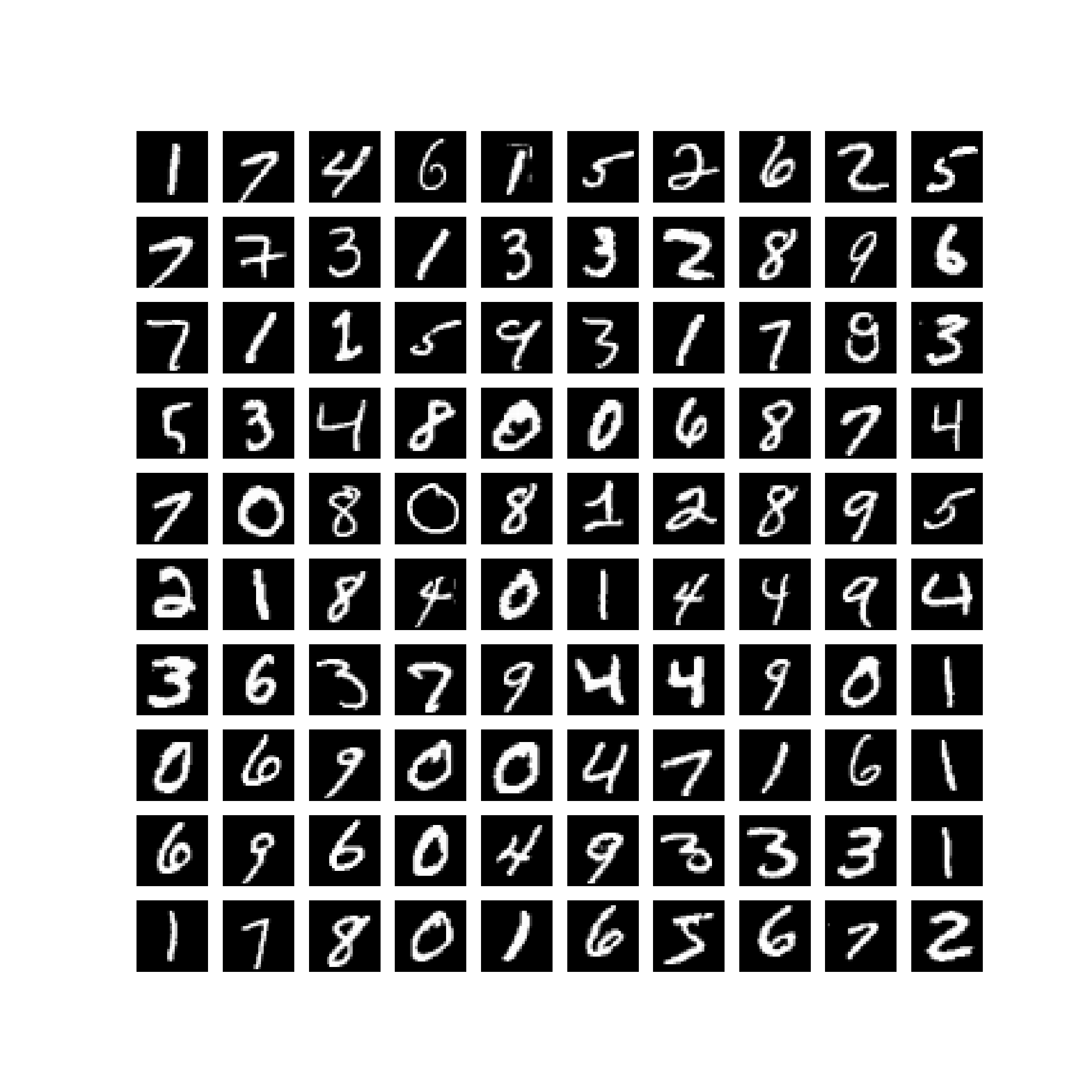}
    \caption{Snapshots of MNIST}
    \label{fig:true_MNIST}
  \end{subfigure}
  \hfill
  \begin{subfigure}[b]{0.32\textwidth}
    \includegraphics[bb=0 0 1300 1000, width=\columnwidth]{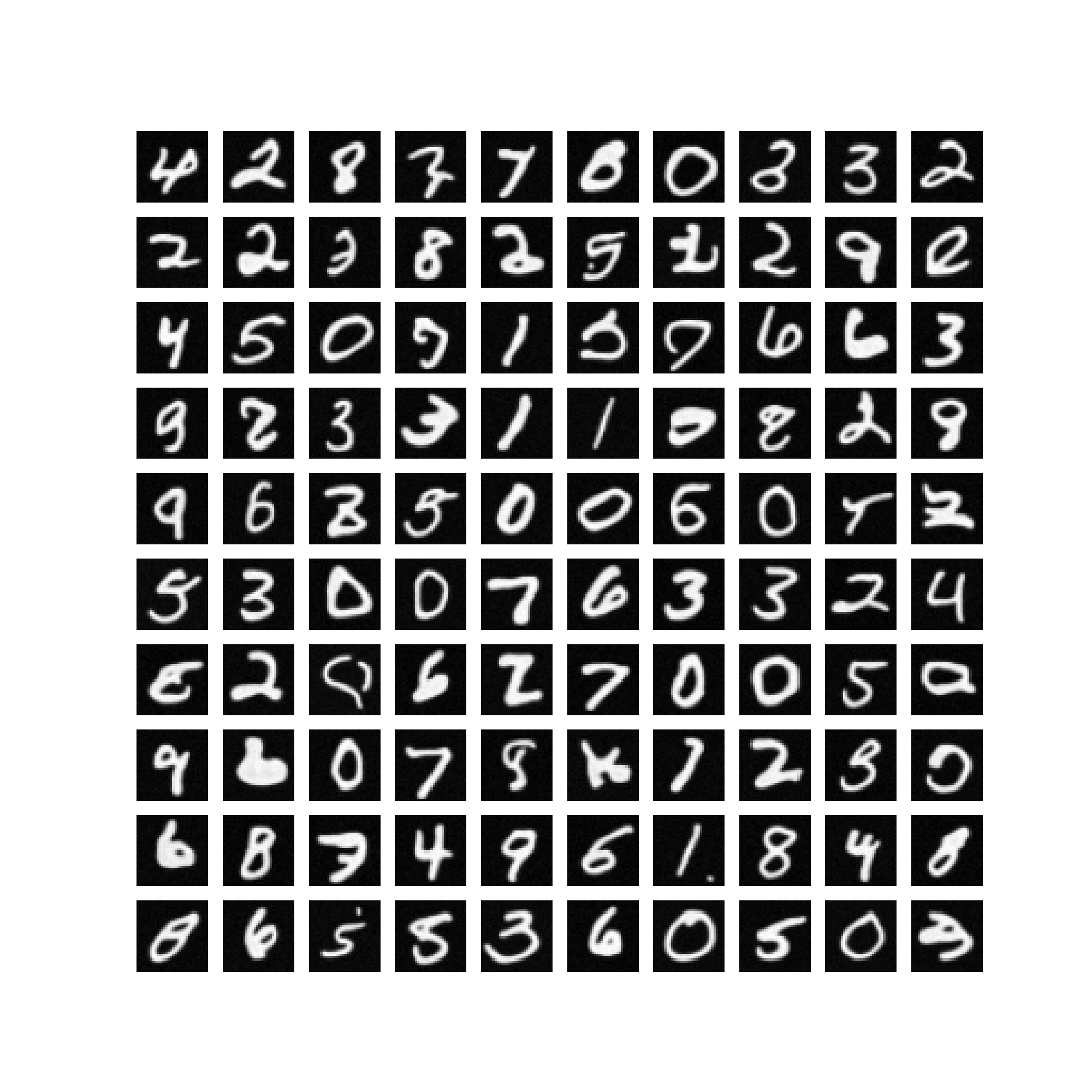}
    \caption{CEM, {FID: 56.55}}
    \label{fig:CEM_MNIST_realTime}
  \end{subfigure}
  \hfill
  \begin{subfigure}[b]{0.32\textwidth}
    \includegraphics[bb=0 0 1300 1000, width=\columnwidth]{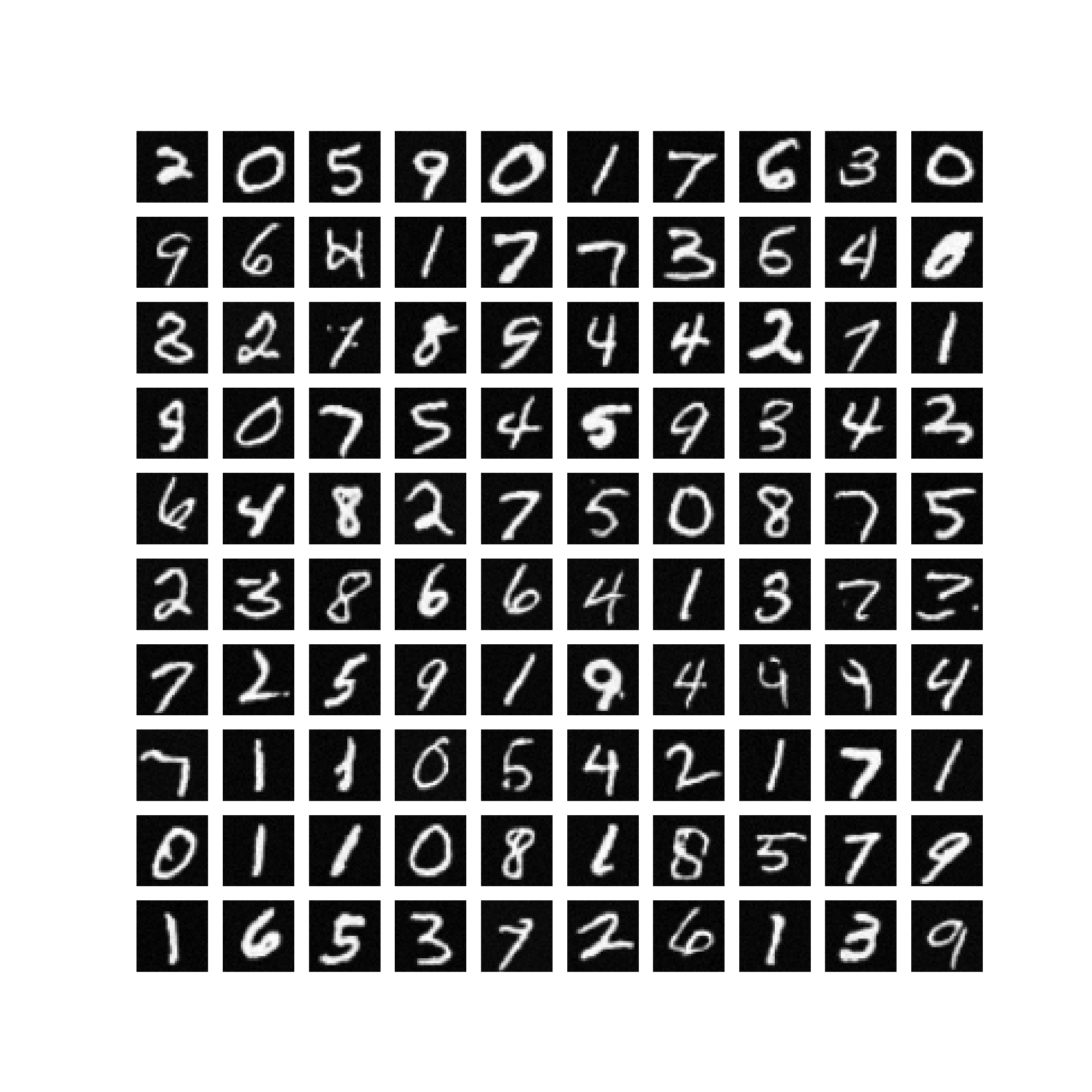}
    \caption{DDPM, {FID:63.64}}
    \label{fig:DDPM_MNIST_realTime}
  \end{subfigure}
  \caption{Performance of CEM and DDPM on MNIST}
  \label{fig:MNIST_realTime}
\end{figure}

\subsection{Ablation Studies}\quad

\noindent\textbf{Impact of sampling schedule.} Theorem \ref{Thm1} shows the general existence of the singularities during the sampling process. An arbitrary sampling schedule may lead to numerical instabilities during solving reverse-time SDEs. To this end, we take $50$ time steps from $T=10$ for the sampling process and compare the linear schedule, quadratic schedule, and the proposed exponential schedule \eqref{exp_law} in Figure \ref{fig:comparison_schedules}. As expected, we can see that the exponential schedule significantly improves the sampling performance of CEM as a result of respecting the growth of the scale of the drift. As an intermediate between linear and exponential, the quadratic schedule yields similar results to the exponential schedule, but with slightly inferior performance.


\begin{figure}[htbp]
\begin{center}
\includegraphics[bb=50 0 600 200, width=0.7\columnwidth]{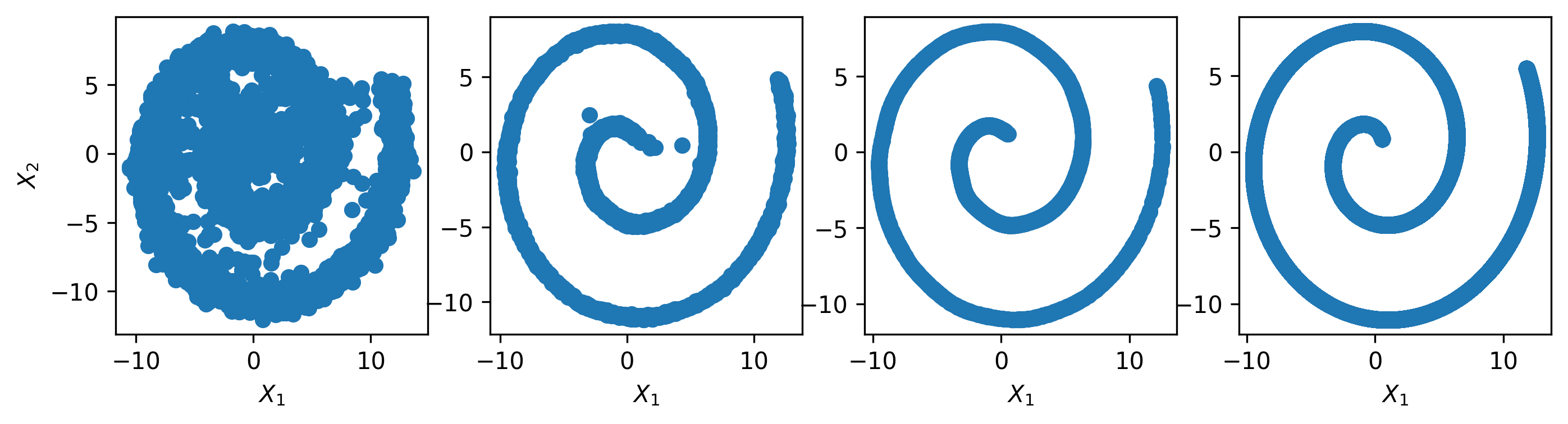}
\end{center}
\caption{Comparison of different sampling schedules. From left to right: linear schedule, quadratic schedule, exponential schedule and ground truth. The network configuration is as follows: 2 hidden layers, each layer with 32 nodes, and \textbf{Tanh} as the activation function.}
\label{fig:comparison_schedules}
\end{figure}

\noindent\textbf{Impact of weighting function $\lambda$.} The weighting function $\lambda$ in \eqref{newscore_loss} is also a major impact factor for the performance and should be carefully designed for the training in order to normalize the training objective. We choose three different weighting functions $\lambda(t)=1,\frac{1}{(e^t-1)^2},\frac{1}{e^t-1}$ and compare their sampling performance in Figure \ref{fig:comparison_weight_func}. We can see  that the proposed weighting function $\lambda(t)=\frac{1}{e^t-1}$ achieves a better sampling result than the other two functions.


\begin{figure}[htbp]
\begin{center}
\includegraphics[bb=50 0 600 200, width=0.7\columnwidth]{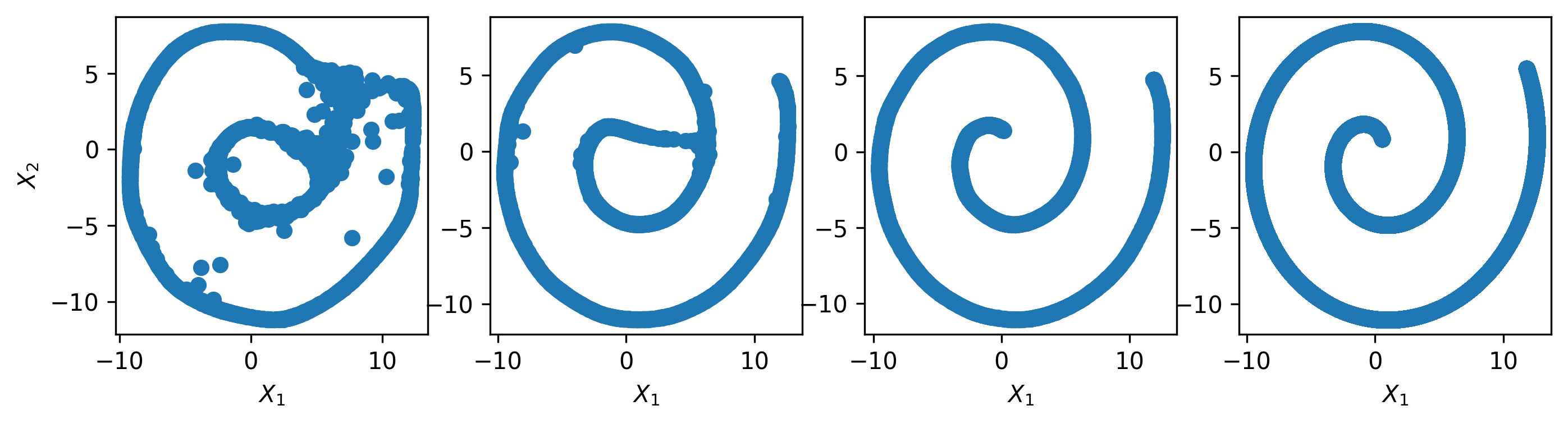}
\end{center}
\caption{Comparison of different weighting functions. From left to right: constant weighting function, $\frac{1}{(e^t-1)^2}$, $\frac{1}{e^t-1}$ and ground truth. The network configuration is as follows: 2 hidden layers, each layer with 32 nodes, and \textbf{Tanh} as the activation function.}
\label{fig:comparison_weight_func}
\end{figure}

\bibliographystyle{plain}

\appendix
\section{Appendix}
\noindent{All codes are available: \url{https://github.com/Yubin-Lu/CEM_Diffusion_Model}}
{
\subsection{Interpretation of training loss}\label{app:training-goal}
Here we provide the justification of \eqref{training-goal}, derivation of goal of DDPM and CEM are similar. We consider a general integrable function $f:\mathbb{R}^d\times \mathbb{R}^d \to \mathbb{R}^d$. $\forall X\in\mathbb{R}^d$ and fixing $X_t=X$. Now as the projection property of conditional expectation, $g^*=E_{X_0|X_t=X}f(X_t,X_0)$ (later it will be a function of $X$) is the minimizer of the function $J'(g)=E_{X_0|X_t=X}\|f(X_t,X_0)-g\|^2$.
Clearly, $\forall X$, the density function of $X_0$ reads $\frac{ p_{X_0,X_t}(X_0,X)}{\int p_{X_0,X_t}(X_0',X)dX_0'}$ and hence,
\begin{align}
    J'(g(X)) =\int \|f(X,X_0)-g(X)\|^2 \frac{ p_{X_0,X_t}(X_0,X)}{\int p_{X_0,X_t}(X_0',X)dX_0'} dX_0.
\end{align}
Note $\int p_{X_0,X_t}(X_0',X)dX_0'$ is a positive function of $X$, then $\forall g:\mathbb{R}^d \to \mathbb{R}^d$,
\begin{align}
 \int J'(g(X)) \int p_{X_0,X_t}(X_0',X)dX_0' dX \geq \int J'(g^*(X)) \int p_{X_0,X_t}(X_0',X)dX_0' dX.
\end{align}
The last one implies, by viewing $g^*$ as a function $\mathbb{R}^d \to \mathbb{R}^d$, $g^*$ minimizes the following functional,
\begin{align}
    &\int J'(g(X)) \int p_{X_0,X_t}(X_0',X)dX_0' dX \nonumber\\
    =& \int \|f(X,X_0)-g(X)\|^2 { p_{X_0,X_t}(X_0,X)} dX_0dX,
\end{align}
which leads to Eq.~\eqref{training-goal}.
}
 \subsection{Normal distribution} \label{Appendix:NormalDist}

Consider a one-dimensional case. If $X_0$ is a normal distribution $\mathcal{N}(\mu,\sigma^2)$,  the marginal density \eqref{marginal_density} becomes
\begin{align}\label{marginal_density_non_standard_normal}
    p(t,X) &= \frac{1}{Z\sqrt{2\pi}\sigma}\int{\rm exp}\Big(-\frac{\|X-X_0e^{-\frac{t}{2}}\|^2}{2(1-e^{-t})}\Big){\rm exp}\big(-\frac{\|X_{0}-\mu\|^{2}}{2\sigma^2}\big)dX_0\nonumber\\
    &=\frac{1}{Z\sqrt{2\pi}\sigma}\int e^{-L(t,X,X_0)}dX_0.
\end{align}
Here the function $L(t,X,X_0)$ is denoted by:
\begin{align}
    L(t,X,X_0)&:=\frac{\|X-X_0e^{-\frac{t}{2}}\|^2}{2(1-e^{-t})}+\frac{\|X_0-\mu\|^2}{2\sigma^2} \nonumber\\
    &=\frac{\sigma^2\|X-X_0e^{-\frac{t}{2}}\|^2+\|X_0-\mu\|^2(1-e^{-t})}{2\sigma^2(1-e^{-t})} \nonumber\\
    &=\frac{A(t)\|X_0\|^2-B(t,X)X_0+\sigma^2\|X\|^2+\mu^2(1-e^{-t})}{2\sigma^2(1-e^{-t})} \nonumber\\
    &=\frac{\sigma^2\|X\|^2+\mu^2(1-e^{-t})}{2\sigma^2(1-e^{-t})}+\frac{A\|X_0-\frac{B}{2A}\|^2-\frac{B^2}{4A}}{2\sigma^2(1-e^{-t})} \nonumber\\
    &=\frac{\sigma^2\|X\|^2+\mu^2(1-e^{-t})}{2\sigma^2(1-e^{-t})}-\frac{B^2}{8A\sigma^2(1-e^{-t})}+\frac{A\|X_0-\frac{B}{2A}\|^2}{2\sigma^2(1-e^{-t})},
\end{align}
where the function $A(t)=\sigma^2e^{-t}+1-e^{-t}$ and $B(t,X)=2X\sigma^2e^{-\frac{t}{2}}+2\mu(1-e^{-t})$. Therefore, we can rewrite the marginal density \eqref{marginal_density_non_standard_normal} as
\begin{align}
    p(t,X)&=\frac{1}{Z\sqrt{2\pi}\sigma}\exp{\big( \frac{B^2}{8A\sigma^2(1-e^{-t})}-\frac{\sigma^2\|X\|^2+\mu^2(1-e^{-t})}{2\sigma^2(1-e^{-t})}\big)}\int\exp{\big(- \frac{A\|X_0-\frac{B}{2A}\|^2}{2\sigma^2(1-e^{-t})}\big)}dX_0 \nonumber\\
    &=\frac{1}{\sqrt{2\pi A}}\exp{\big( \frac{B^2}{8A\sigma^2(1-e^{-t})}-\frac{\sigma^2\|X\|^2+\mu^2(1-e^{-t})}{2\sigma^2(1-e^{-t})}\big)}\frac{\sqrt{A}}{Z\sigma}\int\exp{\big(-\frac{A\|X_0-\frac{B}{2A}\|^2}{2\sigma^2(1-e^{-t})}\big)}dX_0 \nonumber\\
    &=\frac{1}{\sqrt{2\pi A}} \exp{\big( \frac{B^2}{8A\sigma^2(1-e^{-t})}-\frac{\sigma^2\|X\|^2+\mu^2(1-e^{-t})}{2\sigma^2(1-e^{-t})}\big)}.
\end{align}
Subsequently, we have
\begin{equation}\label{log_p_nnrmal}
    \log p(t,X) = \log(\frac{1}{\sqrt{2\pi A}})+\frac{B^2}{8A\sigma^2(1-e^{-t})}-\frac{\sigma^2\|X\|^2+\mu^2(1-e^{-t})}{2\sigma^2(1-e^{-t})}.
\end{equation}
Substituting $A(t)=\sigma^2e^{-t}+1-e^{-t}$ and $B(t,X)=2X\sigma^2e^{-\frac{t}{2}}+2\mu(1-e^{-t})$ into \eqref{log_p_nnrmal} yields
\begin{align}\label{grad_log_p}
    \nabla_X\log p(t,X)&=\frac{B\nabla_X B}{4A\sigma^2(1-e^{-t})}-\frac{X}{1-e^{-t}}\nonumber\\
    &=\frac{4\sigma^4e^{-t}X+4\mu\sigma^2e^{-\frac{t}{2}}(1-e^{-t})}{4\sigma^2(1-e^{-t})(\sigma^2e^{-t}+1-e^{-t})}-\frac{X}{1-e^{-t}}\nonumber\\
    &=\big[\frac{\sigma^2e^{-t}}{(1-e^{-t})(\sigma^2e^{-t}+1-e^{-t})}-\frac{1}{1-e^{-t}}\big]X+\frac{\mu e^{-\frac{t}{2}}}{\sigma^2e^{-t}+1-e^{-t}}\nonumber\\
    &=\frac{-X+\mu e^{-\frac{t}{2}}}{\sigma^2e^{-t}+(1-e^{-t})}.
\end{align}
Thus, the function $\nabla_X\log p(t,X)$ is not singular at $t=0$. This agrees with $\nabla_X\log p(t,X)=-X$ when $\mu=0$ and $\sigma=1$.\\

\subsection{Distribution supported on a low dimensional manifold}\label{eg:lowDmanifold}
 If $X_0$ is a normal distribution $\mathcal{N}(0,1)$ and $Y_0$ is a $\delta_{0}$-distribution,  the marginal density in \eqref{marginal_density} becomes
\begin{align}
    p(t,X,Y)&=\frac{1}{Z}\int\int \exp{-\frac{ (X-X_0e^{-\frac{t}{2}})^2 + (Y-Y_0e^{-\frac{t}{2}})^2}{2(1-e^{-t})}}\rho(X_0,Y_0)dX_0dY_0\nonumber\\
    &=\frac{1}{Z}\int\int \exp{-\frac{ (X-X_0e^{-\frac{t}{2}})^2 + (Y-Y_0e^{-\frac{t}{2}})^2}{2(1-e^{-t})}}\rho(X_0)\delta_0(Y_0)dX_0dY_0\nonumber\\
    &=\frac{1}{Z\sqrt{2\pi}}\int \exp{-\frac{ (X-X_0e^{-\frac{t}{2}})^2 + Y^2}{2(1-e^{-t})}}e^{-\frac{X_0^2}{2}}dX_0\nonumber\\
    &=\frac{1}{Z\sqrt{2\pi}}\int \exp{ -\frac{ (X_0-Xe^{-\frac{t}{2}})^2+X^2+Y^2-X^2e^{-t} }{2(1-e^{-t})} }dX_0\nonumber\\
    &=\frac{\sqrt{1-e^{-t}}}{Z}\exp{\frac{X^2e^{-t}-X^2-Y^2}{2(1-e^{-t})}}.
\end{align}
Therefore,
\begin{equation}
    \log p(t,X,Y) = \log(\frac{\sqrt{1-e^{-t}}}{Z})+\frac{X^2e^{-t}-X^2-Y^2}{2(1-e^{-t})},
\end{equation}
and
\begin{align}
    \nabla_X\log p(t,X,Y) &= -X, \\
    \nabla_Y\log p(t,X,Y) &= -\frac{Y}{1-e^{-t}}.
\end{align}

\subsection{Point cloud distribution}\label{Appendix:pointcloud}
We now derive the analytic expression for $E_{{X_0|X_t}}[X_0|X_t=X]$ when $X_0$ is drawn from a point cloud. Suppose that the number of points is $K$ (denoted by $\{X_0^{(i)}\}_{i=1}^{K}$), then
\begin{align}\label{App:Ex0}
    E_{{X_0|X_t}}[X_0|X_t=X]=\frac{\frac{1}{Z} \int X_0\exp\Big(-\frac{\|X-X_0 e^{-t/2}\|^2}{2(1-e^{-t})}\Big)p_{data}(X_0)dX_0,}{\frac{1}{Z} \int\exp\Big(-\frac{\|X-X_0 e^{-t/2}\|^2}{2(1-e^{-t})}\Big)p_{data}(X_0)dX_0}
    =\frac{\sum_{i=1}^{K}X_{0}^{(i)}\exp{\Big(-\frac{\|X-X_{0}^{(i)}e^{-\frac{t}{2}}\|^2}{2(1-e^{-t})}\Big)}}{\sum_{i=1}^{K}\exp{\Big(-\frac{\|X-X_{0}^{(i)}e^{-\frac{t}{2}}\|^2}{2(1-e^{-t})}\Big)}},
\end{align}
where $Z$ is a normalizing factor that depends on $t$ and the dimension of $X_0$.

\end{document}